\documentclass[10pt,twocolumn,letterpaper]{article}

\usepackage{cvpr}
\usepackage{times}
\usepackage{epsfig}
\usepackage{graphicx}
\usepackage{amsmath}
\usepackage{amssymb}

\usepackage{bm}
\usepackage{array}
\usepackage{amsthm}

\usepackage[compatibility=false]{caption}
\usepackage{subcaption}
\usepackage{multirow}
\usepackage{newfloat}
\DeclareFloatingEnvironment[
fileext=los,
listname=List of Schemes,
name=Scheme,
placement=tbhp,
within=section,
]{scheme}

\graphicspath{{figure/}}

\newtheorem{prop}{Proposition}

\newcolumntype{M}[1]{>{\centering\arraybackslash}m{#1}}

\usepackage[pagebackref=true,breaklinks=true,letterpaper=true,colorlinks,bookmarks=false]{hyperref}

\cvprfinalcopy 

\def\cvprPaperID{2237} 

\ifcvprfinal\fi
\begin{document}

\title{Hardware-Efficient Guided Image Filtering For Multi-Label Problem  \vspace{-0.5cm} }
\author{ Longquan Dai$^1$ \quad  Mengke Yuan$^{2,3}$ \quad Zechao Li$^1$ \quad Xiaopeng Zhang$^2$ \quad Jinhui Tang$^{1, *}$ \\
1. School of Computer Science and Engineering, Nanjing University of Science and Technology \\  2.   National Laboratory of Pattern Recognition, 
 Institute of Automation, Chinese Academy of Sciences\\
 3. University of Chinese Academy of Sciences
\\
{\tt\small \{dailongquan, zechao.li, jinhuitang\}@njust.edu.cn, \{mengke.yuan, xpzhang\}@nlpr.ia.ac.cn  \vspace{-0.5cm}}
}

\maketitle
\renewcommand{\thefootnote}{}
\footnotetext{* Corresponding author.}
\renewcommand{\thefootnote}{\arabic{footnote}}


\begin{abstract}
   The Guided Filter (GF) is well-known for its linear complexity. However, when filtering an image with an n-channel guidance, GF needs to invert an $n \times n$ matrix for each pixel. To the best of our knowledge existing matrix inverse algorithms are inefficient on current hardwares. This shortcoming limits applications of multichannel guidance in computation intensive system such as multi-label system.  \textbf{We need a new GF-like filter that can perform fast multichannel image guided filtering.} Since the optimal linear complexity of GF cannot be minimized further, the only way thus is to bring all potentialities of current parallel computing hardwares into full play. In this paper we propose a hardware-efficient Guided Filter (HGF), which solves the efficiency problem of multichannel guided image filtering and yields competent results when applying it to multi-label problems with synthesized polynomial multichannel guidance. Specifically, in order to boost the filtering  performance,  HGF takes a new matrix inverse algorithm which only involves two hardware-efficient operations: element-wise arithmetic calculations and box filtering. In order to break the linear model restriction, HGF synthesizes a polynomial multichannel guidance to introduce nonlinearity. Benefiting from our polynomial guidance and hardware-efficient matrix inverse algorithm, HGF not only is more sensitive to the underlying structure of guidance but also achieves the fastest computing speed. Due to these merits, HGF obtains state-of-the-art results in terms of accuracy and efficiency in the computation intensive multi-label systems. 
\end{abstract}

\section{Introduction}

Since 2010, GF has been applied to many computer vision and graphics problems such as image retargeting~\cite{Ding_CVPR_2011}, color transfer~\cite{Chia_TOG_2011} and video dehazing~\cite{Zhang_TVC_2011}. 
Among them, multi-label system maybe one of the most suitable applications for GF to make full use of  its efficiency and effectiveness because the heavy computation in multi-label system is in urgent need of a fast filtering tool. A typical multi-label system~\cite{Hosni_PAMI_2013} records the costs $c$ for choosing a label $l$ at coordinates $x$ and $y$ in a cost volume ($V(x, y, l) = c$) according to the input data. Then WTA (Winner-Takes-All) label selection strategy is exploited to determine the final label for each pixel $(x, y)$ after the aggregation (\ie guided image smoothing) step operated on each slice of the cost volume. Besides applying GF to the aggregation step of the multi-label system, we can also incorporate it into MRF models~\cite{Vineet_IJCV_2014, Xiao_ECCV_2006} since edge-aware filters adopted in these MRF models can be taken place by GF directly.  

Due to the linear complexity and edge-preserving ability, GF is consider as the best choice~\cite{Hosni_PAMI_2013} among all candidate filters for the multi-label system. However, a blemish of GF is that the color image guided filtering algorithm is not very efficient. 
More specifically, matrix inverse is a time-consuming operation, but GF needs to invert a $3 \times 3$ matrix for each pixel to calculate both coefficients $\vec{w}_{\mathfrak{p}}$ of Eq \eqref{eq:GF_linear_model} 
\begin{scheme}[h]
	\vspace{-0.3cm}
	\begin{equation}
	\bm{Z}({\mathfrak{q}}) = \sum_{i=1}^{3} \vec{w}_{\mathfrak{p}}(i) \bm{I}_{i}(\mathfrak{q}) + \vec{w}_{\mathfrak{p}}(0)
	\label{eq:GF_linear_model}
	\end{equation}
	\vspace{-0.5cm}
\end{scheme}
and filtering result $\mathbf{Z}(\mathfrak{q})$ according to the $i$th channel $\bm{I}_{i}(\mathfrak{q})$ of a color input guidance $\bm{I}$. Fig~\ref{fig:preface:b} plots the run time of inverting $10^6$ matrices with increasing matrix size, where the matrix inverse algorithm is the built-in LU algorithms of OpenCV. We can observe that the run time increases with the size of matrix dramatically. Hence it is inefficient to apply GF to multi-label system with a multi-channel guidance, especially for a large channel number.


To decrease the execution time, the most straightforward method for GF is to launch a set of threads to invert matrices simultaneously. However this strategy is not very efficient on current hardwares. This is because 1, both CPU and GPU rely on SIMD (Single Instruction Multiple Data) architecture to boost performance; 2, branch instructions are inevitable for traditional matrix inversion methods such as LU algorithm~\cite{Hougardy_BOOK_2016}; 3, SIMD architecture cannot achieve the fastest speed to run branch instructions as these instructions require decomposing each vector into elements and processing them sequentially on the architecture.  Fig~\ref{fig:preface:b} also illustrates the run time of inverting $10^6$ matrices simultaneously. The reported data is unsatisfactory. To avoid branch instructions, we can invert matrices according to the analytic solution of matrix inverse~\cite{Horn_BOOK_1985}. The fastest OpenCV implementation of GF takes this strategy to invert $3 \times 3$ matrices and successfully reduces the run time of inverting $10^6$ matrices to less than $100ms$.  But the implementation complexity of the analytic solution increases with matrix size heavily. When the size of matrix becomes large, the method is no longer implementable manually. This explains why we miss the run time of this method in Fig~\ref{fig:preface:b}.

\begin{figure}[t]
	\centering		\vspace{-0.2cm}
	\begin{subfigure}[b]{0.45\linewidth}
		\includegraphics[width=\textwidth]{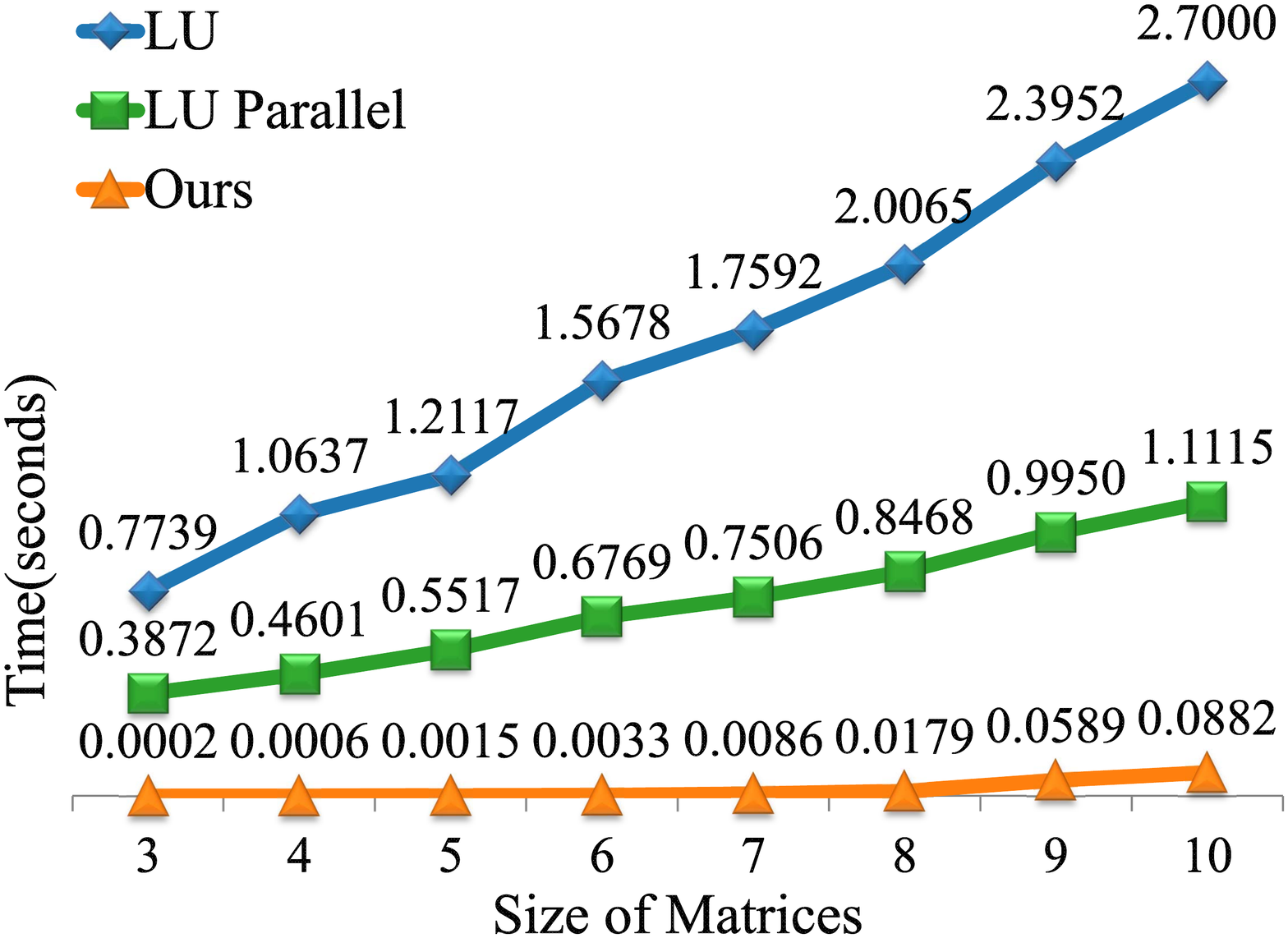}
		\caption{Matrix Inverse}
		\label{fig:preface:b}
	\end{subfigure}	
	\begin{subfigure}[b]{0.45\linewidth}
		\includegraphics[width=\textwidth]{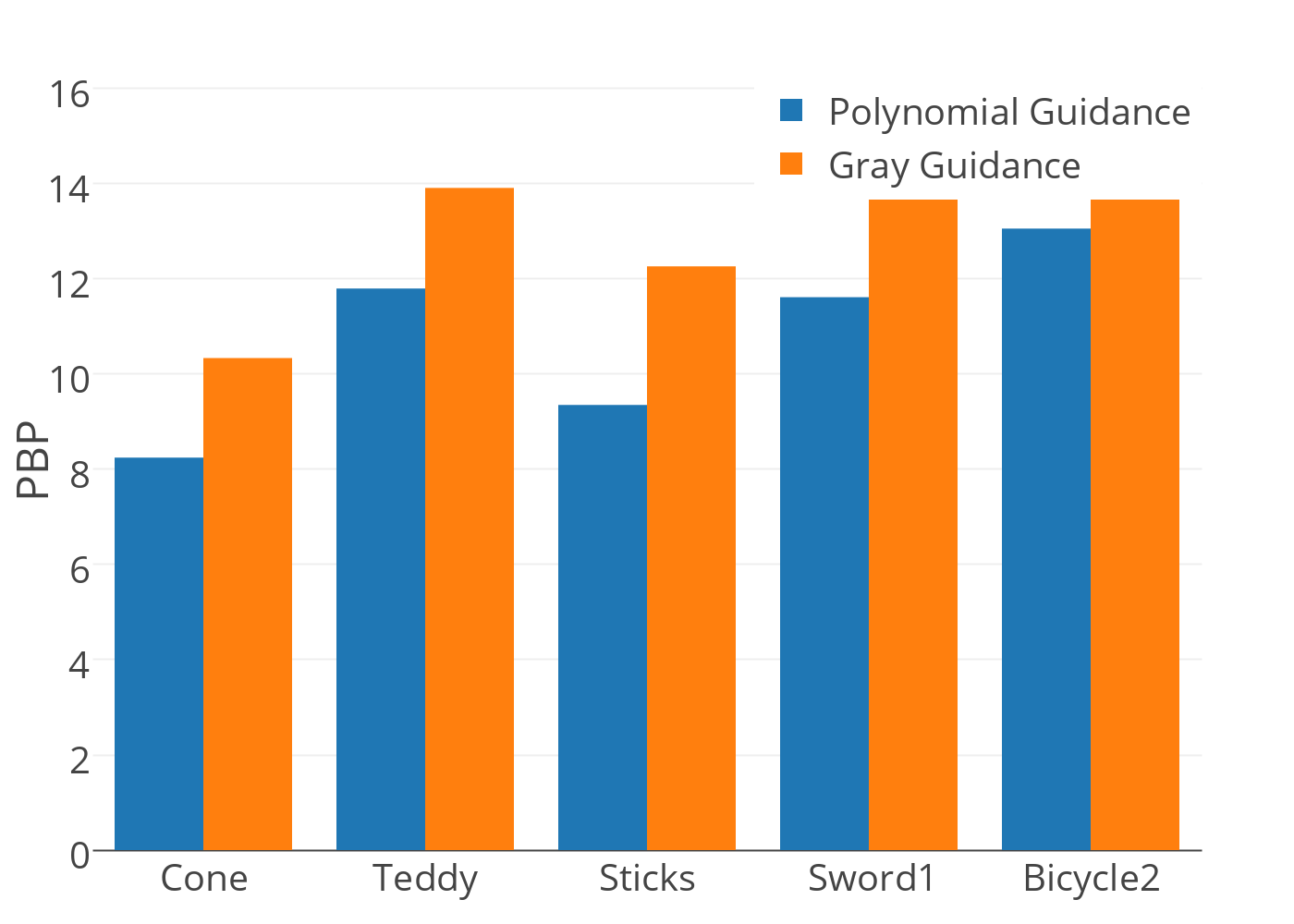}
		\caption{Stereo Matching}
		\label{fig:preface:a}
	\end{subfigure}
	\caption{\textbf{Matrix Inverse $\bm{\&}$ Stereo Matching:} (a) demonstrates the run time of inverting $10^6$ matrices with increasing size (b) illustrates the stereo matching results of GF with different guidances. }
	\label{fig:preface}   \vspace{-0.5cm}
\end{figure}

\textbf{We propose a hardware-efficient matrix inverse algorithm} to exploit the parallel computing power of hardwares to its full potential. The ability of our algorithm stems from two hardware-efficient operations: \emph{element-wise arithmetic calculations} and \emph{box filtering}. To compare with the built-in LU algorithm adopted by OpenCV, we collect run times of our algorithm implemented by Intel OpenCL and plot them in Fig~\ref{fig:preface:b}. The result proves that our method gets a substantial performance boost on Intel platform.


\textbf{We introduce nonlinearity to GF by synthesizing a polynomial multi-channel guidance from the input guidance} to conquer the drawback of the linear model of GF. The linear model usually underfits the 
input data due to its simplicity and thus is apt to produce over-smoothed results. Fig~\ref{fig:preface:a} shows the stereo matching results under the guidance of a gray image and our  synthesized multi-channel image. Obviously, the synthesized guidance produces better results. Note that the extra
run time in accompany with the nonlinear model can be reduced to an acceptable level by our hardware-efficient matrix inverse algorithm.



\textbf{Taking proposed polynomial guidance and new matrix inverse algorithm, we construct a hardware-efficient guided filter (HGF) for multi-label problem}. For presentation clarity, we organize the rest of this paper in following way. Section~\ref{sec:work} is devoted to introduce related work. In the next section, we present how to invert matrix efficiently and its application in computing regression coefficients. Section~\ref{sec:HGF},~\ref{sec:discussion} discuss HGF in details, including the hardware-efficient implementation and the method to synthesis polynomial guidances. Finally, we conduct comprehensive experiments to verify the superiority of HGF in section~\ref{sec:exp}.

\section{Related Work}
\label{sec:work}

\textbf{Acceleration:} People have already proposed many image guided filters such as the Bilateral Filter (BF)~\cite{Tomasi_ICCV_1998} and the Median Filter (MF)~\cite{Zhang_CVPR_2014} for edge-preserving smoothing. However, the benefit is not free. For example, the naive implementation of BF usually requires several minutes to filter a typical megapixel image. Several methods have been developed to accelerate the computation of bilateral filtering  ~\cite{Durand_TOG_2002,Paris_IJCV_2009,Porikli_CVPR_2008,Yang_CVPR_2009,Gunturk_TIP_2011,Dai_TIP_2016} and as a result people obtained several orders of magnitude speedup. Similar to BF, MF can also  effectively filter images while not strongly blurring edges. Although there is little work for speeding up weighted MF~\cite{Ma_ICCV_2013, Zhang_CVPR_2014}, unweighted MF has several solutions~\cite{Cline_ICIP_2007,Kass_TOG_2010,Perreault_TIP_2007,Weiss_TOG_2006}. Generally, the main idea of above accelerating methods focuses on reducing the computational complexity per pixel to the optimal $O(1)$. However, the minimal linear complexity of GF leaves no room to reduce the computational complexity further. To accelerate GF, He \etal~\cite{He_CoRR_2015} remind that GF can be sped up by subsampling to subset the amount of data processed. But this strategy inevitably introduces approximation error. A different way is to implement the algorithm of GF on VLSI directly~\cite{Kao_TCSVT_2014}. But Kao \etal~\cite{Kao_TCSVT_2014} only give  a gray guidance based GF implementation as existing matrix inverse algorithms are too complicated to implement on hardware. Different from other algorithms, our hardware-efficient matrix inverse algorithm not only can fully exploit current parallel  processors
 but also can be easily implemented.

\textbf{Improvement:} Due to nice visual quality, fast speed, and ease of implementation, GF since its birth has received much effort to make it much better. In 2012,  Lu \etal~\cite{Lu_CVPR_2012} pointed out that the box window adopted by GF is not geometric-adaptive and design a adaptive support region to take place of it. However, Tan \etal~\cite{Tan_CVPR_2014} argued that the new designed adaptive support region depends on the scanning order and put forward a further improvement for the support region. In addition, they introduced a quadratic spatial regularization term to GF. But the filtering speed is very slow because they employ traditional algorithms to invert $5 \times 5$ matrices. Later, Dai \etal~\cite{Dai_ICCV_2015} successfully incorporated spatial similarity into GF without significant performance degeneration. Unlike previous methods, Li \etal~\cite{Li_TIP_2015} modified GF to adapt it to suppress halo. 
Through careful investigation, we can find that all these improvements do not make any modification for the gray/color guidance or incorporate more complicated nonlinear model to GF to achieve better results. This is because these modifications will increase the matrix size and therefore make the run time unacceptable. Note that even though all of them only take into account the linear model, their speeds are still slower than GF~\cite{Dai_ICCV_2015}. Benefiting from our hardware-efficient matrix inverse algorithm, we realizes these modifications in HGF without sacrificing efficiency.

\section{Technology Brief}
\label{sec:brief}

In this section, we outline our matrix inverse algorithm and its application in computing $\vec{w}_{\mathfrak{p}}$ in Eq~\eqref{eq:HGF_w1}.  Before the formal discussion, we kindly remind readers to refer to Table~\ref{tab:symbol_conventions} for symbol conventions. 

The key step of HGF is to calculate the vector $\vec{w}_{\mathfrak{p}}$ according to Eq~\eqref{eq:HGF_w1} which contains the matrix inverse operation $(\lambda \bm{E} + \bm{X}^T_{\mathfrak{p}} \bm{X}_{\mathfrak{p}} )^{-1}$.
\begin{figure}[h]
	\vspace{-0.1cm}
	\begin{equation}
	\vec{w}_{\mathfrak{p}} = (\lambda \bm{E} + \bm{X}^T_{\mathfrak{p}} \bm{X}_{\mathfrak{p}} )^{-1}\bm{X}^T_{\mathfrak{p}} \vec{c}_{n+1, \mathfrak{p}}
	\label{eq:HGF_w1}
	\end{equation}
	\vspace{-0.5cm}
\end{figure}
Here $\bm{E}$ denotes an identity matrix and $\bm{X}_{\mathfrak{p}} = [\vec{c}_{0, \mathfrak{p}}, \cdots, \vec{c}_{n, \mathfrak{p}}] $. As for the vector $\vec{c}_{i, \mathfrak{p}}$,  we put $\vec{c}_{i, \mathfrak{p}} = [\bm{G}_{i}(\mathfrak{q}_1),\cdots, \bm{G}_{i }(\mathfrak{q}_{|\Omega_{\mathfrak{p}}|})]^T (0  \leq i \leq  n+1) $ to record values $\bm{G}_i(\mathfrak{q}_k)$, where  $ \mathfrak{q}_k \in \Omega_{\mathfrak{p}}$,  $\Omega_{\mathfrak{p}}$ presents a neighborhood centered at the pixel $\mathfrak{p}$, $|\Omega_{\mathfrak{p}}|$ indicates the total number of pixels in $\Omega_{\mathfrak{p}}$ and $\bm{G}_i$ $(0  \leq i \leq  n+1)$ is an image.

To invert $\lambda \bm{E} + \bm{X}^T_{\mathfrak{p}} \bm{X}_{\mathfrak{p}}$ efficiently, we will substitute  $\lambda \bm{E} + \bm{X}^T_{\mathfrak{p}} \bm{X}_{\mathfrak{p}}$ with $\lambda \bm{E} + \sum_{i=0}^{n}  \vec{c}_{i, \mathfrak{p}} \vec{c}_{i,\mathfrak{p}}^T $ by reformulating Eq~\eqref{eq:HGF_w1} as Eq~\eqref{eq:HGF_w2}.
\begin{figure}[h]
	\vspace{-0.3cm}
\begin{align}
\vec{w}_{\mathfrak{p}} & =  \bm{X}^T_{\mathfrak{p}} (\lambda \bm{E} + \bm{X}_{\mathfrak{p}} \bm{X}^T_{\mathfrak{p}} )^{-1} \vec{c}_{n+1, \mathfrak{p}} \nonumber \\
& =  [\vec{c}_{0, \mathfrak{p}}^T, \cdots, \vec{c}_{n, \mathfrak{p}}^T]^T (\lambda \bm{E} + \sum_{i=0}^{n}  \vec{c}_{i, \mathfrak{p}} \vec{c}_{i,\mathfrak{p}}^T )^{-1} \vec{c}_{n+1, \mathfrak{p}}  \label{eq:HGF_w2}
\end{align}
\vspace{-0.5cm}
\end{figure}
This is because $(\lambda \bm{E} + \sum_{i=0}^{n}  \vec{c}_{i, \mathfrak{p}} \vec{c}_{i,\mathfrak{p}}^T )^{-1}$ can be interpreted as a linear combination of outer products $\vec{c}_{i, \mathfrak{p}} \vec{c}_{i,\mathfrak{p}}^T$~\eqref{eq:HGF_inverse} according to following Proposition~\ref{prop:inverse}.
\begin{prop} If $\lambda \bm{E} + \sum_{i=0}^{n}  \vec{c}_{i, \mathfrak{p}} \vec{c}_{i,\mathfrak{p}}^T $ is invertible, then
	\begin{equation}
	(\lambda \bm{E} + \sum_{i=0}^{n}  \vec{c}_{i, \mathfrak{p}} \vec{c}_{i,\mathfrak{p}}^T)^{-1} = \lambda^{-1} \bm{E} + \sum_{i,j=0}^{n} \alpha_{ij, \mathfrak{p}} \vec{c}_{i, \mathfrak{p}} \vec{c}_{j,\mathfrak{p}}^T
	\label{eq:HGF_inverse}
	\end{equation}	
	where $\alpha_{ij, \mathfrak{p}} = \alpha^{n}_{ij, \mathfrak{p}}$ which can be iteratively computed by
	
	\begin{align}\hspace{-2mm}
	\alpha_{ij,\mathfrak{p}}^{\kappa} = & \begin{cases}
	\gamma_{\mathfrak{p}}^\kappa F_{ij,\mathfrak{p}}^{\kappa} + \alpha_{ij,\mathfrak{p}}^{\kappa-1} & i < \kappa, j < \kappa \\
	\lambda^{-1} \gamma_{\mathfrak{p}}^\kappa (\sum_{n=0}^{\kappa-1} \alpha^{\kappa-1}_{in, \mathfrak{p}} G_{n j, \mathfrak{p}})  &  i < \kappa, j = \kappa \\
	\lambda^{-1} \gamma_{\mathfrak{p}}^\kappa (\sum_{m=0}^{\kappa-1} \alpha^{\kappa-1}_{mj, \mathfrak{p}} G_{i m, \mathfrak{p}} ) &  i = \kappa, j < \kappa  \\
	\lambda^{-2} \gamma_{\mathfrak{p}}^\kappa & i = j = \kappa 
	\end{cases}
	\end{align} 
	from $\kappa = 1$ to $\kappa = n$ with $\alpha_{00,\mathfrak{p}}^{0} = - (\lambda +  G_{00,\mathfrak{p}})^{-1}$, $G_{ij, \mathfrak{p}} = \vec{c}_{i,\mathfrak{p}}^T \vec{c}_{j, \mathfrak{p}}$, $F_{ij,\mathfrak{p}}^{\kappa} = \sum_{m,n=0}^{\kappa-1} \alpha^{\kappa-1}_{im,\mathfrak{p}} \alpha^{\kappa-1}_{nj,\mathfrak{p}} G_{m\kappa,\mathfrak{p}}  G_{\kappa n,\mathfrak{p}}$ and $\gamma_{\mathfrak{p}}^\kappa = - ( 1 + \lambda^{-1} G_{\kappa \kappa,\mathfrak{p}} + \sum_{m,n=0}^{\kappa-1} \alpha^{\kappa-1}_{mn, \mathfrak{p}}  G_{\kappa m, \mathfrak{p}} G_{n\kappa, \mathfrak{p}} )^{-1}  $. 
	\label{prop:inverse}
\end{prop}
\begin{proof}
The recursive formula is derived from the Sherman-Morrison formula~\cite{Sherman_AMS_1950}. Constrained by maximal 8 pages constraint, we leave derivation details to supplementary materials. 
\end{proof}

\begin{table}[t]
	\begin{tabular}{|c|p{5.3cm}|}
		\hline 
		\multirow{2}{*}{$\mathcal{F}$}    & Calligraphic means that $\mathcal{F}$ is a function. \tabularnewline
		\hline 
		\multirow{2}{*}{$i$}  & Italics emphasizes that the variable $i$ is a scalar.\tabularnewline
		\hline 
		$\mathfrak{p}$ & Fraktur denotes $\mathfrak{p}$ is an image pixel.\tabularnewline
		\hline 
		$\vec{v}$ & Arrow $\vec{\quad}$  imlies $v$ is a vector.\tabularnewline
		\hline 
		\multirow{2}{*}{$\bm{X}$} & Bold (disregarding case) indicates $\bm{X}$ is a matrix.\tabularnewline
		\hline 
		\multirow{2}{*}{$i_\mathfrak{p} \backslash \vec{v}_{\mathfrak{p}}\backslash\bm{X}_{\mathfrak{p}}$} & The subscript $_{\mathfrak{p}}$ suggests $i_\mathfrak{p}\backslash \vec{v}_{\mathfrak{p}}\backslash\bm{X}_{\mathfrak{p}}$
		is the scalar$\backslash$vector$\backslash$matrix located at $\mathfrak{p}$. \tabularnewline
		\hline 
		\multirow{2}{*}{$\vec{v}_{i,\mathfrak{p}}\backslash$ $\bm{X}_{i,\mathfrak{p}}$} & The subscript $_{i,\mathfrak{p}}$ suggests $\vec{v}_{i,\mathfrak{p}}\backslash$
		$\bm{X}_{i,\mathfrak{p}}$ is the $i$th vector$\backslash$matrix of $\mathfrak{p}$. \tabularnewline
		\hline 
		\multirow{2}{*}{$\vec{v}(j)\backslash\vec{v}_{\mathfrak{p}}(j)\backslash\vec{v}_{i,\mathfrak{p}}(j)$}   & $\vec{v}(j)\backslash\vec{v}_{\mathfrak{p}}(j)\backslash\vec{v}_{i,\mathfrak{p}}(j)$ is the $j$th element of $\vec{v}$ $\backslash$ $\vec{v}_{\mathfrak{p}}$$\backslash$$\vec{v}_{i,\mathfrak{p}}$
		.\tabularnewline
		\hline 
		\multirow{2}{*}{$\bm{X}(\mathfrak{p})\backslash\bm{X}_{i}(\mathfrak{p})$}  & $\bm{X}(\mathfrak{p})\backslash\bm{X}_{i}(\mathfrak{p})$ is the element of $\bm{X}\backslash\bm{X}_{i}$
		located at $\mathfrak{p}$. \tabularnewline
		\hline 
	\end{tabular}
\caption{Symbol Conventions}
\label{tab:symbol_conventions}
\end{table}

Further, putting Eq~\eqref{eq:HGF_inverse} into Eq~\eqref{eq:HGF_w2},
we are able to transform the $k$th element $\vec{w}_{\mathfrak{p}}(k)$ of $\vec{w}_{\mathfrak{p}}$ to a linear combination of $G_{ij, \mathfrak{p}}$ as illustrated in Eq~\eqref{eq:HGF_w3}.
According to Proposition~\ref{prop:box}, 
\begin{figure}[h]
	\vspace{-0.5cm}
	\begin{align}
	\vec{w}_{\mathfrak{p}}(k) = & \vec{c}_{k, \mathfrak{p}}^T (\lambda \bm{E} + \sum_{i=0}^{n}  \vec{c}_{i, \mathfrak{p}} \vec{c}_{i,\mathfrak{p}}^T )^{-1} \vec{c}_{n+1, \mathfrak{p}} \nonumber \\
	= & \vec{c}_{k, \mathfrak{p}}^T  ( \lambda^{-1} \bm{E} + \sum_{i,j=0}^{n} \alpha_{ij, \mathfrak{p}} \vec{c}_{i, \mathfrak{p}} \vec{c}_{j,\mathfrak{p}}^T ) \vec{c}_{n+1, \mathfrak{p}} \nonumber \\
	= & \lambda^{-1} I_{k\, n+1, \mathfrak{p}} + \sum_{i,j=0}^{n} \alpha_{ij, \mathfrak{p}} I_{ki, \mathfrak{p} } I_{j\, n+1, \mathfrak{p} }
	\label{eq:HGF_w3}
	\end{align}
	\vspace{-0.7cm}
\end{figure}
$G_{ij,\mathfrak{p}}$ equals to the box filtering result $\bm{G}_{ij}(\mathfrak{p})$ of the element-wise production image $\bm{G}_i\bm{G}_j$. Then $\vec{w}_{\mathfrak{p}}(k)$ can be synthesized by composing a set of box filtering results.

\begin{prop}
if the neighborhood $\Omega_{\mathfrak{p}}$ of $\mathfrak{p}$ is a box window, we have $G_{ij,\mathfrak{p}} = \bm{G}_{ij}(\mathfrak{p})$ which is the box filtering result of the element-wise production image $\bm{G}_i\bm{G}_j$ in $\Omega_{\mathfrak{p}}$.
\label{prop:box}
\end{prop}
\begin{proof}
We have the dot production $G_{ij,\mathfrak{p}} = \vec{c}_{i,\mathfrak{p}}^T \vec{c}_{j, \mathfrak{p}} = \sum_{k=1}^{|\Omega_\mathfrak{p}|} \bm{G}_{i}(\mathfrak{q}_k) \bm{G}_{j}(\mathfrak{q}_k) =  \sum_{\mathfrak{q} \in \Omega_{\mathfrak{p}}} \bm{G}_{i}(\mathfrak{q}) \bm{G}_{j}(\mathfrak{q}) $ and
the box filtering result $\bm{G}_{ij}(\mathfrak{p}) = \sum_{\mathfrak{q} \in \Omega_{\mathfrak{p}}} \bm{G}_i\bm{G}_j(\mathfrak{q}) = \sum_{\mathfrak{q} \in \Omega_{\mathfrak{p}}} \bm{G}_i(\mathfrak{q})\bm{G}_j(\mathfrak{q}) $. So we conclude $G_{ij,\mathfrak{p}} = \bm{G}_{ij}(\mathfrak{p})$.
\end{proof}

\begin{figure*}[t]
	\vspace{-0.2cm}
	\begin{subfigure}[b]{\linewidth}
		\includegraphics[width=\textwidth, height=3.7cm]{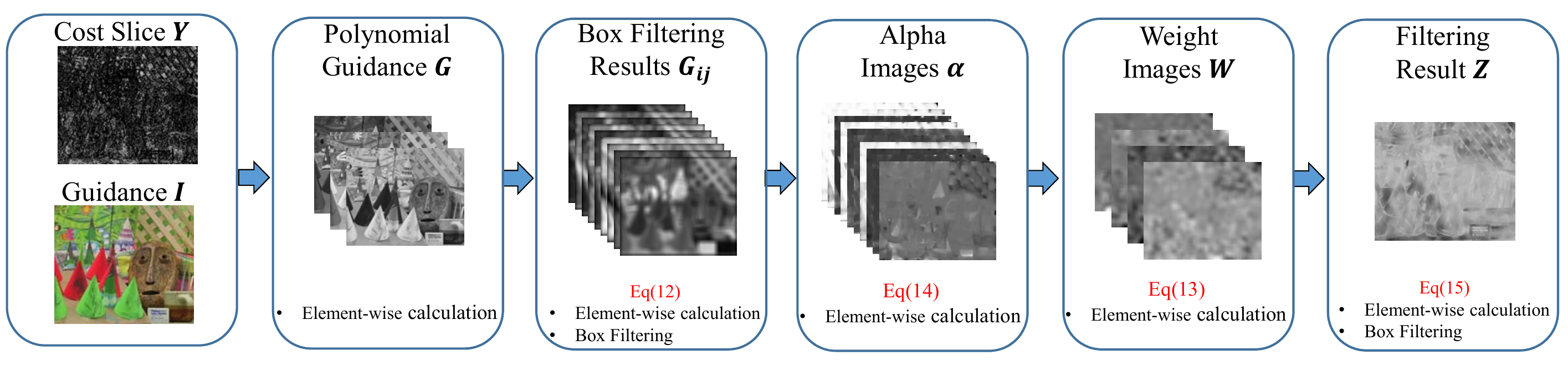}
	\end{subfigure}	
\caption{\textbf{Filtering Flowchart of HGF with the Input from a Multi-Label System}.  At first, we exploit the input guidance to synthesize a polynomial guidance. Then the input cost slice together with the polynomial guidance are used to produce box filtering images $\bm{G}_{ij}$~\eqref{eq:HGF_Iij}. After that we compute $\bm{\alpha}_{ij}$ and $\bm{W}_i$ from $\bm{G}_{ij}$ according to Eq~\eqref{eq:alpha}~\eqref{eq:HGF_w4}. At last, the filtering result is yielded by Eq~\eqref{eq:HGF_z2}. Note that each step only involves element-wise arithmetic computation or box filtering. }
\label{fig:flowchart}
\end{figure*}
\noindent Above all, the $k$th element $\vec{w}_{\mathfrak{p}}(k)$ of $\vec{w}_{\mathfrak{p}}$ can be figured out in two steps:
\begin{enumerate}  
	\item applying box filter to element-wise production image $\bm{G}_i\bm{G}_j$ for $0 \leq i, j \leq n+1$ to produce $\bm{G}_{ij}$; 
	\item computing $\vec{w}_{\mathfrak{p}}(k)$ according to the linear combination~\eqref{eq:HGF_w3} of $\bm{G}_{ij}$. 
\end{enumerate}
Then applying the procedure to each $\vec{w}_{\mathfrak{p}}(k)$, we will compose the vector $\vec{w}_{\mathfrak{p}}$. At last, we note that above two steps computing procedure is only composed by arithmetic calculation and box filtering. The matrix inverse operation is completely eliminated in the computation procedure.

\section{Hardware-Efficient Guided Filter}
\label{sec:HGF}

In this section, we define HGF and present its hardware-efficient implementation in detail.  Fig~\ref{fig:flowchart} plots the overall computing procedure to give a general idea for our HGF.

\subsection{Defintion}

The key assumption of HGF is the generalized linear model~\eqref{eq:HGF_linear_output_model} between the synthesized $n$-channel guidance $\bm{G}$ and the output $\bm{Z}$, where $\vec{w}_{\mathfrak{p}}$ is constant in the window $\Omega_{\mathfrak{p}}$. 
\begin{equation}
\bm{Z}(\mathfrak{q}) = \sum_{i=1}^{n} \vec{w}_{ \mathfrak{p}}(i) \bm{G}_{i}(\mathfrak{q}) + \vec{w}_{\mathfrak{p}}(0), \; \forall \mathfrak{q} \in \Omega_{\mathfrak{p}}
\label{eq:HGF_linear_output_model}
\end{equation} 

According to the local multipoint filtering framework~\cite{Katkovnik_IJCV_2010}, HGF comprises two major steps: 1) \emph{multipoint estimation:} calculating the estimates for a set of points within a local support, and 2) \emph{aggregation:} fusing multipoint estimates available for each point. Specifically, in the first step, HGF estimates the coefficients $\vec{w}_{\mathfrak{p}}$ of model \eqref{eq:HGF_linear_output_model} by minimizing the linear ridge regression~\eqref{eq:HGF_linear_regression}, where $\bm{Y}$ denotes an input image and the closed-form solution of the minimizer $\vec{w}_{\mathfrak{p}}$ is Eq~\eqref{eq:HGF_w1}.
\begin{figure}[h]
\vspace{-0.3cm}
\begin{equation}\hspace{-0.2cm}
\min_{\vec{w}_{\mathfrak{p}}} \lambda \| \vec{w}_{\mathfrak{p}} \|_2^2  +  \sum_{\mathfrak{q} \in \Omega_{\mathfrak{p}}} ( \bm{Y}(\mathfrak{q}) - \sum_{i=1}^{n} \vec{w}_{\mathfrak{p}}(i) \bm{G}_{i}(\mathfrak{q}) - \vec{w}_{\mathfrak{p}}(0) )^2
\label{eq:HGF_linear_regression}
\end{equation}
\vspace{-0.5cm}
\end{figure}
In the second step, HGF puts the minimizer $\vec{w}_{\mathfrak{p}}$ of optimization~\eqref{eq:HGF_linear_regression} into Eq~\eqref{eq:HGF_linear_output_model} and obtains a set of values $\bm{Z}'_{\mathfrak{p}}(\mathfrak{q}) = \sum_{i=1}^{n} \vec{w}_{i, \mathfrak{p}} \bm{G}_{i}(\mathfrak{q}) + \vec{w}_{0, \mathfrak{p}}, \mathfrak{q} \in \Omega_{\mathfrak{p}}$for a given window $\Omega_{\mathfrak{p}}$. So each pixel $\mathfrak{q}$ has $| \Omega_{\mathfrak{p}} |$ values. HGF aggregates these values together and considers their mean $ \frac{1}{|\Omega_{\mathfrak{q}}|} \sum_{\mathfrak{p} \in \Omega_{\mathfrak{q}}} \bm{Z}'_{\mathfrak{p}}(\mathfrak{q})$ as final filtering result $\bm{Z}(\mathfrak{q})$. So we have
\begin{align}
\bm{Z}(\mathfrak{q}) 
= & \frac{1}{|\Omega_{\mathfrak{q}}|} \sum_{\mathfrak{p} \in \Omega_{\mathfrak{q}}} \sum_{i=1}^{n} \vec{w}_{ \mathfrak{p}}(i) \bm{G}_{i}(\mathfrak{q}) + \vec{w}_{\mathfrak{p}}(0) \nonumber \\
= & \sum_{i=1}^{n} \vec{w}^a_{\mathfrak{q}}(i) \bm{G}_{i}(\mathfrak{q}) + \vec{w}^a_{\mathfrak{q}}(0)
\label{eq:HGF_z1}
\end{align}
where $\vec{w}^a_{\mathfrak{q}} = \frac{1}{|\Omega_{\mathfrak{q}}|} \sum_{\mathfrak{p} \in \Omega_{\mathfrak{q}}} \vec{w}_{ \mathfrak{p}} $ is the average of $\vec{w}_{ \mathfrak{p}}$ in $\Omega_{\mathfrak{q}}$.

\subsection{Polynomial Guidance}

GF exploits the linear model~\eqref{eq:GF_linear_model} to estimate filtering results from an input guidance. However, it is unreasonable to expect that the linear model works well under all circumstances as the linear models~\eqref{eq:GF_linear_model} of GF are apt to underfits data due to its simplicity and thus produces over-smoothing filtering results.  We address the problem by synthesizing a polynomial guidance from a raw image.


The polynomial guidance is designed to introduce nonlinearity to the generalized linear model~\eqref{eq:HGF_linear_output_model}.  For example, Eq~\eqref{eq:HGF_input_output_polynomial_model} shows
\begin{figure}[h]
	\vspace{-0.3cm}
	\begin{equation}
	\bm{Z}(\mathfrak{q}) = \sum_{i=1}^{d} \vec{w}_{ \mathfrak{p}}(i) \bm{I}^{i}(\mathfrak{q}) + \vec{w}_{\mathfrak{p}}(0), \; \forall \mathfrak{q} \in \Omega_{\mathfrak{p}}
	\label{eq:HGF_input_output_polynomial_model}
	\end{equation}
	\vspace{-0.3cm}
\end{figure}
the polynomial model with a gray input  guidance $\bm{I}$, 
where $d$ is the degree of the polynomial function. Compared with the linear model~\eqref{eq:HGF_linear_output_model}, we can find the equivalence between the linear model~\eqref{eq:HGF_linear_output_model} and the polynomial model~\eqref{eq:HGF_input_output_polynomial_model} if we assume $\bm{G}_i = \bm{I}^i$. This inspires us to 
introduce nonlinearity to HGF by synthesizing a polynomial guidance, where each channel of the synthesized polynomial guidance is produced by a polynomial function operated on a channel of the input guidance.

This method can be easily extended to the multichannel guidance. In the case when the input guidance is multichannel, it is straightforward to apply the map procedure $\bm{G}_{(i-1)d+j} = \bm{I}_i^{j}$ to each channel independently, where $\bm{I}_i$ denotes the $i$th channel of the multichannel guidance $\bm{I}$ and $n$ is the channel number of $\bm{I}$. After that we stack produced results of each channel of the input multichannel guidance $\bm{I}$ to synthesize our polynomial guidance. Mathematically, the linear model~\eqref{eq:HGF_linear_output_model} in this situation is equivalent to following nonlinear polynomial model~\eqref{eq:HGF_input_output_polynomial_model_color}. So we successfully endow the nonliearity to the generalized linear model~\eqref{eq:HGF_linear_output_model} of HGF.
\begin{figure}[h]
	\vspace{-0.5cm}
	\begin{equation}
	\bm{Z}(\mathfrak{q}) = \sum_{i=1}^{n}\sum_{j=1}^{d} \vec{w}_{ \mathfrak{p}}((i-1)d+j) \bm{I}^{j}_i(\mathfrak{q}) + \vec{w}_{\mathfrak{p}}(0) 
	\label{eq:HGF_input_output_polynomial_model_color}
	\end{equation}
	\vspace{-0.5cm}
\end{figure}

\subsection{Hardware-Efficient Implementation}


Here we present the matrix inverse algorithm tailored to $(\lambda \bm{E} + \bm{X}^T_{\mathfrak{p}} \bm{X}_{\mathfrak{p}} )^{-1}$ to make full use of the parallel computing ability of hardwares. 
HGF in the first step calculates the minimizer~\eqref{eq:HGF_w1} of the linear ridge regression~\eqref{eq:HGF_linear_regression}. However traditional matrix inverse algorithms cannot fully take advantages of the parallel computing hardwares. Fortunately, Eq~\eqref{eq:HGF_w3} sheds light on a hardware-efficient way to compute $\vec{w}_{\mathfrak{p}}$ because it guarantees the $k$th element $\vec{w}_{\mathfrak{p}}(k)$ of $\vec{w}_{\mathfrak{p}}$ is the linear combination of the box filtering result $\bm{G}_{ij}(\mathfrak{p})$. Specifically, putting $\mathcal{B}(\bm{X})$ denote the box filtering result of an image $\bm{X}$, $\bm{W}_i$ and $\bm{\alpha}_{ij}$ record all values of $\vec{w}_{\mathfrak{p}}(i)$ and $\alpha_{ij,\mathfrak{p}}$ for arbitrary $\mathfrak{p}$ in the image domain (\ie $\bm{W}_i(\mathfrak{p}) = \vec{w}_{\mathfrak{p}}(i)$, $\bm{\alpha}_{ij} (\mathfrak{p}) = \alpha_{ij,\mathfrak{p}}$),  we can generalize Eq~\eqref{eq:HGF_w3} to following element-wise arithmetic calculations~\eqref{eq:HGF_w4} of box filtering results $\bm{G}_{ij}$~\eqref{eq:HGF_Iij}. Here $\bm{G}_0$ denotes matrix of ones, $\bm{G}_{i} (1\leq i \leq n)$ stand for the $i$th channel of the synthesized polynomial $n$-channel guidance $\bm{G}$, $\bm{G}_{n+1}$ be the alias of the input image $\bm{Y}$.
\begin{figure}[h]
\vspace{-0.6cm}
\begin{align}
\bm{G}_{ij} &= \mathcal{B}(\bm{G}_{i} \bm{G}_{j})  \label{eq:HGF_Iij} \\
\bm{W}_i &= \lambda^{-1} \bm{G}_{\kappa n+1} + \sum_{i,j=0}^{n} \bm{\alpha}_{ij} \bm{G}_{\kappa i} \bm{G}_{j n+1}
\label{eq:HGF_w4}
\end{align}
\vspace{-0.6cm}
\end{figure}
Similarly, the updating formulate of $\alpha_{ij,\mathfrak{p}}^{\kappa}$ can also  be modified to 
the element-wise arithmetic calculation of matrices~\eqref{eq:alpha}, 
\begin{figure}[h]
	\vspace{-0.5cm}
	\begin{equation}
	\bm{\alpha}_{ij}^{\kappa} = \begin{cases}
	\bm{F}_{\mathfrak{p}}^{\kappa} + \bm{\alpha}_{ij}^{\kappa-1}  & i < \kappa, j < \kappa \\
	\lambda^{-1}  \bm{\gamma}^\kappa \sum_{n=0}^{\kappa-1} \bm{\alpha}_{in}^{\kappa-1} \bm{G}_{n\kappa}  &  i < \kappa, j = \kappa \\
	\lambda^{-1} \bm{\gamma}^\kappa \sum_{m=0}^{\kappa-1} \bm{\alpha}_{mj}^{\kappa-1} \bm{G}_{\kappa m}  &  i = \kappa, j < \kappa  \\
	\lambda^{-2} \bm{\gamma}^\kappa & i = j = \kappa 
	\end{cases}
	\label{eq:alpha}
	\end{equation}
	\vspace{-0.5cm}
\end{figure}
where
$\bm{\alpha}_{00}^{0} = - (\lambda + \bm{G}_{00})^{-1}$, 
  $\bm{F}^{\kappa} = \sum_{m,n=0}^{\kappa-1} \bm{\alpha}^{\kappa-1}_{im} \bm{\alpha}^{\kappa-1}_{nj} \bm{G}_{m \kappa}  \bm{G}_{\kappa n}$ and $\bm{\gamma}^\kappa = - (1 + \lambda^{-1} \bm{G}_{\kappa\kappa} + \sum \nolimits_{m,n=0}^{\kappa-1} \bm{\alpha}^{\kappa-1}_{mn}  \bm{G}_{\kappa m}  \bm{G}_{n \kappa} )^{-1}$.

HGF in the second step computes filtering results $\bm{Z}(\mathfrak{p})$ according to the average of coefficients $\vec{w}_{\mathfrak{p}}$. Define the average operator $\mathcal{A}(\bm{X}) = \mathcal{B}(\bm{X}) / \mathcal{B}(\bm{G}_0)$, we can formulate the element-wise arithmetic calculation form of Eq~\eqref{eq:HGF_z1} as
\vspace{-0.2cm}
\begin{equation}
\bm{Z} = \sum_{i=1}^{n} \mathcal{A}(\bm{W}_i) \bm{G}_i + \mathcal{A}(\bm{W}_0)  
\label{eq:HGF_z2} 
\end{equation}

Taking a closer look at Eq~\eqref{eq:HGF_Iij}~\eqref{eq:HGF_w4}~\eqref{eq:alpha}~\eqref{eq:HGF_z2}, we find that all equations only involve two computation types (Fig~\ref{fig:flowchart} sums up  computation types in each step and overall computation flowchart): one is  element-wise arithmetic calculations of matrices, the other is box filtering for images. Both computations can easily exert the parallelism power of current hardware. More importantly, there are highly optimized libraries for CPU and GPU. In the next, we present how to implement them in a hardware-efficient way. 

\vspace{1mm}
\emph{\textbf{Element-wise arithmetic calculations}} is a typical data parallel task because it distributes the operation over elements of a matrix. In parallel computing literatures, it belongs to the map pattern~\cite{McCool_BOOK_2012} which applies an elemental function to an actual collection of input data. Current CPU and GPU do well in these calculations as the SIMD instructions of CPU and GPU can benefit from same computations done on different pieces of data in the data parallel task. Although inverting matrices simultaneously belongs to the map pattern too, it cannot be as efficient as our method because inverting operation is not directly supported by hardware and thus can not benefit from  SIMD instructions. In contrast the arithmetic computations can be directly allocated to a core of CPU or thread of GPU for parallel computing. Finally, we note that element-wise arithmetic calculations are supported by many softwares or libraries such as 
such as Matlab, ViennaCL~\cite{Rupp:ViennaCL} and Arrayfire~\cite{Yalamanchili_CODE_2015}.

\vspace{1mm}
\emph{\textbf{Box filtering}} produces a smoothing image of which has a value equal to the sum of its neighboring pixels in the input image and can be computed in linear time from the Summed Area Table~\cite{Crow_SGF_1984}. Different from element-wise calculation, box filtering belongs to the stencil pattern~\cite{McCool_BOOK_2012} that is a generalization of the map pattern in which an elemental function can access not only a single element in an input collection but also a set of neighbors. A detailed description on how to implement the stencil pattern efficiently on parallel devices can be found in the book~\cite{McCool_BOOK_2012}. Fortunately, we do not need to implement the box filtering manually since two libraries  OpenCV of Intel and NPP of Nvidia already have offered.

\section{More Discussion for HGF}
\label{sec:discussion}

In this section, we clarity the connection and differences between HGF and GF. Moreover, the performance and implementation concerns of HGF will also be discussed.

\subsection{Connection with GF}

Our matrix inverse algorithm can also be used to speed up GF. Specifically, GF minimizes optimization~\eqref{eq:GF_linear_regression} to determine the coefficient $\vec{w}_{\mathfrak{p}}$, where  $\vec{w}_{\mathfrak{p}}$ is a $4$ elements vector.
\begin{figure}[h]
	\vspace{-0.6cm}
	\begin{equation} 
	\min_{\vec{w}_{\mathfrak{p}}  } \lambda  \sum_{i=1}^{3} 
	\vec{w}_{\mathfrak{p}}^2(i)   +  \sum_{\mathfrak{q} \in \Omega_{\mathfrak{p}}} ( \bm{Y}(\mathfrak{q}) - \sum_{i=1}^{3} \vec{w}_{\mathfrak{p}}(i) \bm{I}_{i}(\mathfrak{p}) - \vec{w}_{\mathfrak{p}}(0) )^2
	\label{eq:GF_linear_regression}
	\end{equation}
	\vspace{-0.8cm}
\end{figure}
Let $\vec{c}_{i, \mathfrak{p}}$ be a $n$ element vector extracted from $\bm{I}_i$ in $\Omega_{\mathfrak{p}}$, $\bm{X}_{\mathfrak{p}} = [\vec{c}_{1, \mathfrak{p}}, \vec{c}_{2, \mathfrak{p}}, \vec{c}_{3, \mathfrak{p}}] $, $\vec{\bm{1}}$ denote the identity vector,  $\vec{x}_{\mathfrak{p}} =  \frac{1}{n} \bm{X}^T \vec{\bm{1}}$, $c_{n+1, \mathfrak{p}} = \frac{1}{n} \vec{c}_{n+1, \mathfrak{p}}^T  \vec{\bm{1}} $, $\bm{X}'^T_{\mathfrak{p}} = \bm{X}^T_{\mathfrak{p}} - \vec{1} \vec{x}^T_{\mathfrak{p}} $ and $\vec{c}'_{n+1, \mathfrak{p}} = \vec{c}_{n+1, \mathfrak{p}} - c_{n+1, \mathfrak{p}} \vec{\bm{1}}$, 
 then  $
\vec{w}_{\mathfrak{p}}(0) = c_{n+1, \mathfrak{p}} - \vec{w}^T_{\mathfrak{p}} \vec{x}_{\mathfrak{p}} $.
\begin{figure}[h]\vspace{-0.5cm}
	\begin{equation} 
     [\vec{w}_{\mathfrak{p}}(1), \vec{w}_{\mathfrak{p}}(2), \vec{w}_{\mathfrak{p}}(3)]^T = (\lambda \bm{E} + \bm{X}'^T_{\mathfrak{p}} \bm{X}'_{\mathfrak{p}} )^{-1}\bm{X}'^T_{\mathfrak{p}} \vec{c}'_{n+1, \mathfrak{p}}
\label{eq:GF_w}
	\end{equation}
	\vspace{-0.6cm}
\end{figure}
and the last three terms of    $\vec{w}_{\mathfrak{p}}$ can be computed by Eq~\eqref{eq:GF_w}. 
The form of Eq~\eqref{eq:GF_w} is same to Eq~\eqref{eq:HGF_w1}. Hence we can apply the same technique to invert the matrix.

\subsection{Differences From GF}

HGF is similar to GF but not same to it. Simply put, there are three major differences between the two filters:

\emph{\textbf{The guidance is different}}.  GF directly takes a gray/color image  as its guidance $\bm{I}$. Contrarily, HGF synthesizes a multi-channels polynomial guidance $\bm{G}$ from a gray or color input image. With increasing channel number, the guidance carries more and more information.

\emph{\textbf{The cost function is different}}. The regularization term of GF does not punish $\vec{w}_{\mathfrak{p}}(0)$. In contrast, HGF treats all elements of $\vec{w}_{\mathfrak{p}}$ equally in the cost function~\eqref{eq:HGF_linear_regression} as the minimizer~\eqref{eq:HGF_w2} can be computed in a hardware-efficient way and the final result of HGF in the multi-label system is much better according to our experiments.

\emph{\textbf{The matrix inverse algorithm is different}}. GF does not take any special measure to accelerate the speed of inverting matrices. On the contrary, HGF designs an effective matrix inverse algorithm to fully utilize the parallel ability of current hardwares. As a result, the computational procedure of HGF only comprises two hardware-efficient operations: element-wise arithmetic calculations and  box filtering.

\subsection{Performance Evaluation}

\begin{table}[b]
	\vspace{-0.2cm}
	\centering
	\begin{tabular}{|c|c|c|c|c|}
		\hline
		Algorithms & $3$ & $5$ & $7$ & $9$ \\
		\hline \hline 
		$\text{GF}_1$	&     $0.85s$           &      $1.36s$          &        $2.01s$        &        $2.83s$        \\
		$\text{GF}_2$	&     $0.17s$           &      $0.35s$          &        $\diagup$        &        $\diagup$        \\
		HGF	&     $0.08s$           &      $0.16s$          &        $0.25s$        &        $0.43s$        \\
		\hline 
	\end{tabular}
	\caption{Execution time to filter $10^3 \times 10^3$ gray images with different guidances on CPU, where the channel number of guidance varies from $3$ to $10$. }
	\label{tab:time_matlab}
\end{table}

Our experiments are conducted on a Laptop with an Intel i7 CPU and GTX 960 GPU. 
In order to compare the running performance fairly and comprehensively, we exploit two C++ libraries (\ie OpenCV and Arrayfire~\cite{Yalamanchili_CODE_2015}) to implement GF and HGF, where the Summed Area Table~\cite{Crow_SGF_1984} accelerated box filtering and the three matrix inverse algorithms are incorporated into the two filters. Here the three matrix inverse algorithms are:
\begin{itemize}   \vspace{-0.5mm}
	\item \textbf{Built-in LU matrix inverse algorithm:} We use the built-in matrix inverse function ``inv'' of OpenCV to compute final filtering results of GF. The built-in function takes a highly optimized LU matrix inverse algorithm and the input matrix can be arbitrary size.
	\item \textbf{Analytic solution of matrix inverse algorithm:} The closed-form solution of each invertible matrix can be expressed by adjugate matrices. It is an efficient way to calculate the inverse of small matrices but is inefficient for large matrices.  The built-in function ``GuidedFilter'' of OpenCV takes it to compute color-guidance results of GF. We only extend it to inverse $5 \times 5$ matrices because the matrix inverse code becomes much complex  when the matrix size is large than $5$ and thus it is nearly impossible to write the code of inverting $7 \times 7$ and $9 \times 9$ matrices by hand (the analytic solutions of $7 \times 7$ and $9 \times 9$ matrix inverse are listed in the supplemental material for demonstration).
	\item \textbf{Our hardware-efficient matrix inverse algorithm:} We incorporate it with our HGF to produce final results. Note that unlike the first algorithm that can invert arbitrary matrices and the second algorithm that only suits to a small matrix,  our method can invert a  arbitrary size matrix only if the matrix is invertible and can be presented by $\lambda \bm{E} + \sum_{i=0}^{n}  \vec{c}_{i, \mathfrak{p}} \vec{c}_{i,\mathfrak{p}}^T $.
\end{itemize}

OpenCV is taken to reimplement GF according to its public Matlab code, where the C++ code is basically a line-by-line translation of the Matlab code and the matrix inverse algorithm is provided by the built-in ``inv'' function of OpenCV and our implementation according to the analytic solution of matrix inverse. We denote GF adopting the two implementations as $\text{GF}_1$ and $\text{GF}_2$. Since both the algorithm of GF and its Matlab code published by He \etal~\cite{He_PAMI_2013} do not consider parallelism or SIMD instructions, we compile the C++ code of GF with default complier option. Similarly, we do not take tricks to speed up HGF: the code of HGF  is also compiled by default compiler option and all parallel computing infrastructure is provided by Arrayfire.

\begin{figure}[t]
	\centering	
	\begin{subfigure}[b]{0.48\linewidth}
		\includegraphics[width=\textwidth]{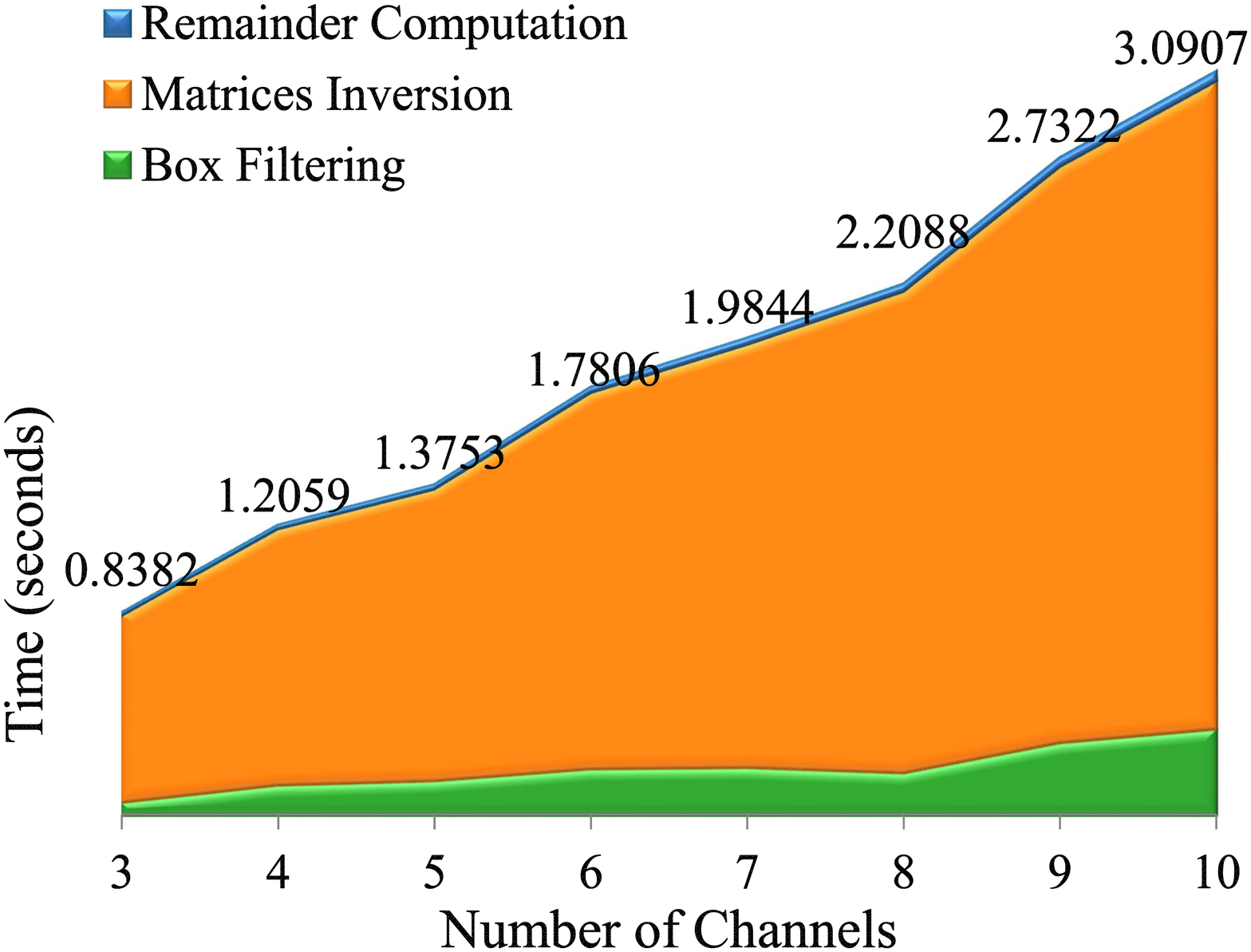}
		\caption{GF}
	\end{subfigure}	
	\begin{subfigure}[b]{0.48\linewidth}
		\includegraphics[width=\textwidth]{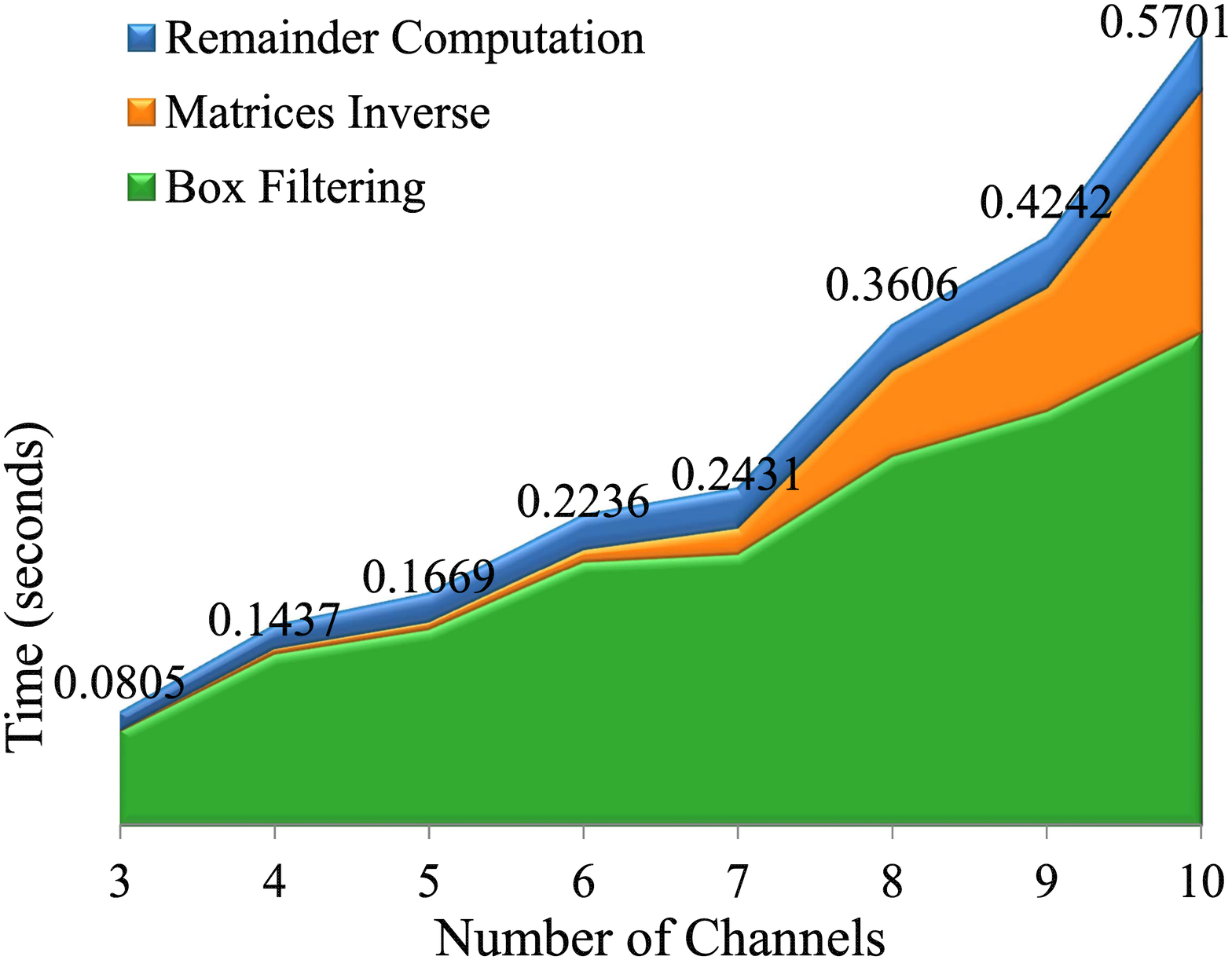}
		\caption{HGF}
	\end{subfigure}	
	\caption{\textbf{Execution Time Anatomy:} (a) (b) illustrate run times of the box filtering in preparing $\lambda \bm{E} + \bm{X}^T_{\mathfrak{p}} \bm{X}_{\mathfrak{p}} $, inverting this matrix and remainder computation of GF and HGF, where the green, orange and blue regions denote the run time of box filtering, matrix inverse and the remainder computation. }
	\label{fig:contribution}
\end{figure}

\begin{figure*}[t]
	\centering	
	\begin{subfigure}[b]{0.16\linewidth}
		\includegraphics[width=\textwidth]{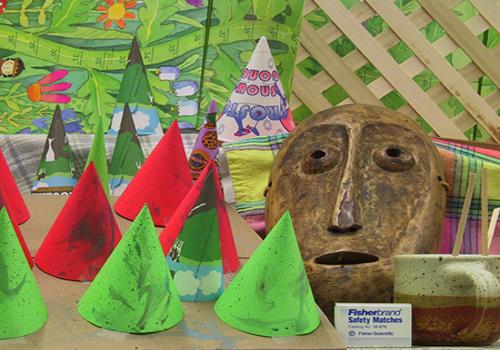}
	\end{subfigure}	
	\begin{subfigure}[b]{0.16\linewidth}
		\includegraphics[width=\textwidth]{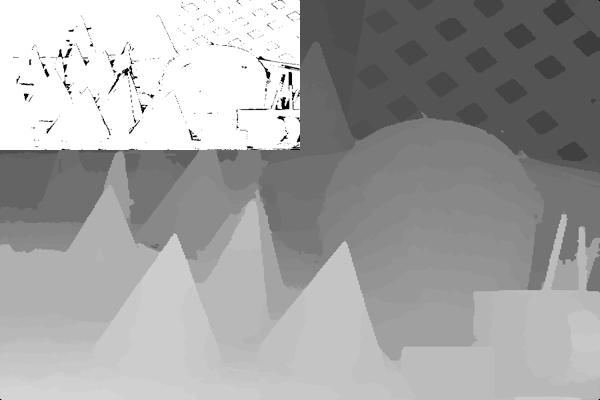}
	\end{subfigure}	
	\begin{subfigure}[b]{0.16\linewidth}
		\includegraphics[width=\textwidth]{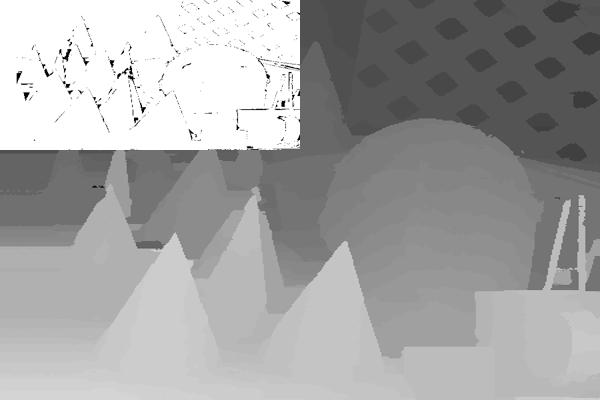}
	\end{subfigure}	
	\begin{subfigure}[b]{0.16\linewidth}
		\includegraphics[width=\textwidth]{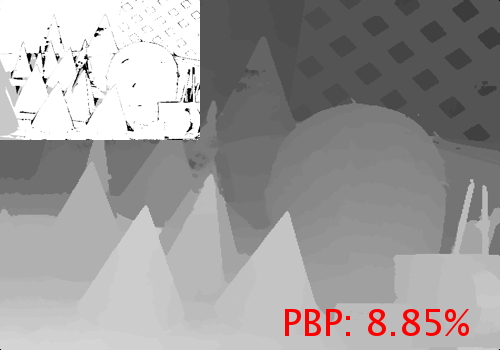}
	\end{subfigure}	
	\begin{subfigure}[b]{0.16\linewidth}
		\includegraphics[width=\textwidth]{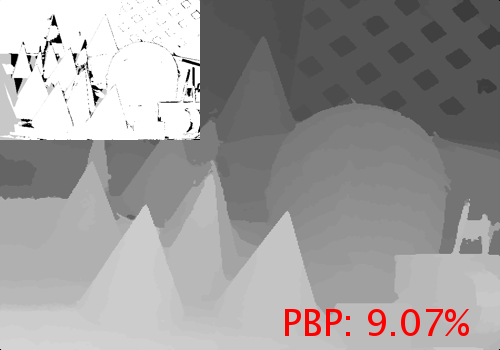}
	\end{subfigure}	
	\begin{subfigure}[b]{0.16\linewidth}
		\includegraphics[width=\textwidth]{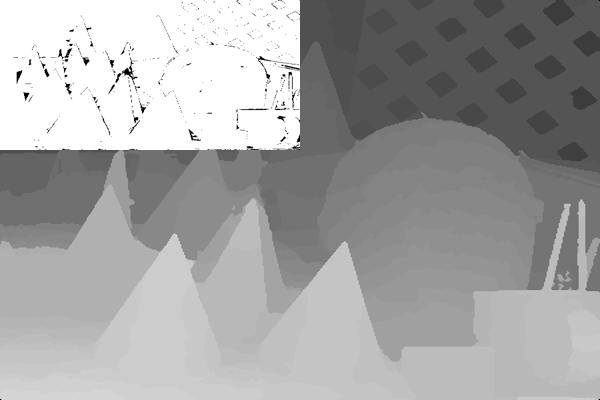}
	\end{subfigure}

	\begin{subfigure}[b]{0.16\linewidth}
		\includegraphics[width=\textwidth]{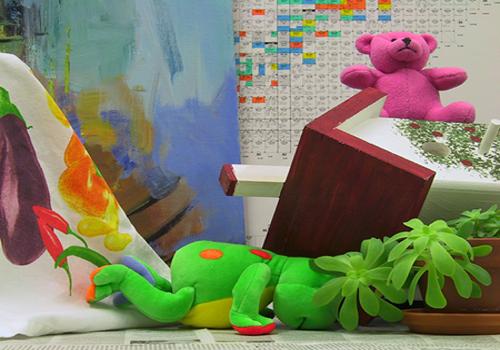}
		\caption{Input}
	\end{subfigure}	
	\begin{subfigure}[b]{0.16\linewidth}
		\includegraphics[width=\textwidth]{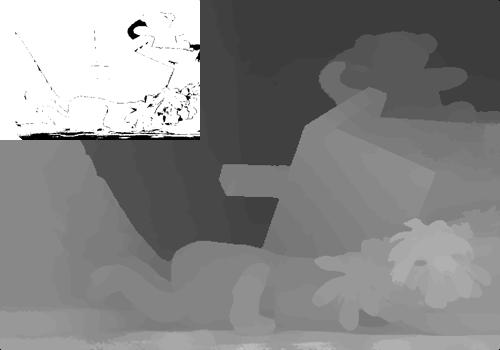}
		\caption{BF}
	\end{subfigure}	
	\begin{subfigure}[b]{0.16\linewidth}
		\includegraphics[width=\textwidth]{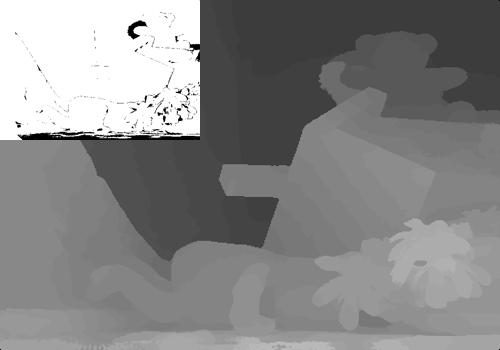}
		\caption{GF}
	\end{subfigure}	
	\begin{subfigure}[b]{0.16\linewidth}
		\includegraphics[width=\textwidth]{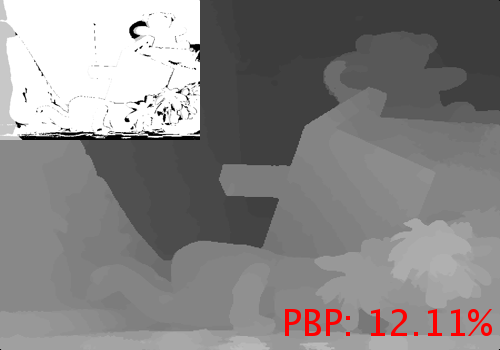}
		\caption{CLMF}
	\end{subfigure}	
	\begin{subfigure}[b]{0.16\linewidth}
		\includegraphics[width=\textwidth]{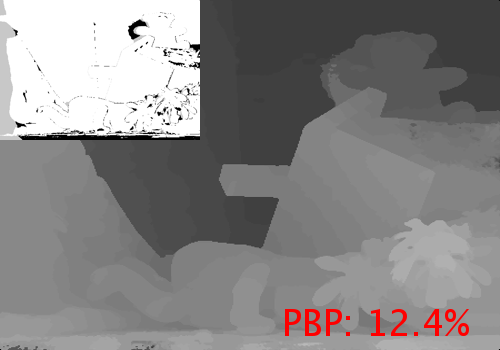}
		\caption{MLPA}
	\end{subfigure}	
	\begin{subfigure}[b]{0.16\linewidth}
		\includegraphics[width=\textwidth]{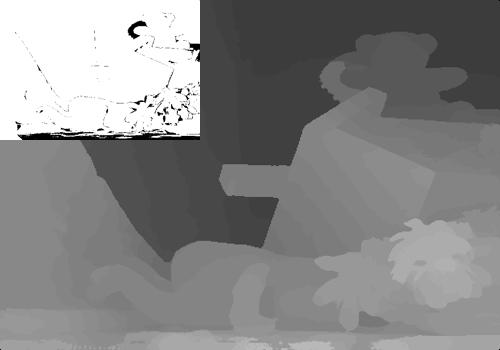}
		\caption{HGF}
	\end{subfigure}	
	\setcounter{figure}{5}
	\caption{\textbf{Stereo Matching Results:} (a) is input color image.  (b) (c) (d) (e) and (f) illustrate the stereo matching results of BF, GF, CLMF, MLPA and HGF, respectively. In the left upper corners of (b) (c) (d) (e) and (f), we show comparisons with the respective groundtruth (error larger than one pixel in nonoccluded are denoted by black and in occlude regions are denoted by gray) and the Percentage of Bad Pixel (PBP) scores are listed in the bottom of each image. }
	\label{fig:stereo}
\end{figure*}

The $\text{GF}_1$ and HGF rows of Table~\ref{tab:time_matlab} report execution times of GF and HGF for a $10^3 \times 10^3$ guidance with increasing channel number on CPU. According to our experiments, HGF is much faster than GF. Due to the smaller running cost of our matrix inverse algorithm, it is more efficient even if channel number of guidance is large. We also test the GF code using analytic solution of matrix inverse. The experimental data for $3/5$-channel guidances are reported in the $\text{GF}_2$ row. We can observe the speed of the analytic solution is faster than generalized matrix inverse function ``inv'' function of OpenCV, but is still slower than our matrix inverse algorithm. In short, the analytic solution achieves a relative fast speed by paying an extremely heavy cost to translate it into element-wise calculation. 




\begin{table}[t]
	\hspace{-0.2cm}
	\begin{tabular}{|c|c|c|c|c|}
		\hline
		Size/$n$             & 3    & 5    & 7    & 9    \\ \hline
		$250 \times 250$   & $3.9ms$ & $10.7ms$ & $39.9ms$ & $110.2ms$ \\ \hline
		$500 \times 500$   & $3.7ms$ & $11.1ms$ & $40.1ms$ & $112.4ms$ \\ \hline
		$1000 \times 1000$ & $4.1ms$ & $10.9ms$ & $41.3ms$ & $115.9ms$ \\ \hline
		$2000 \times 2000$ & $4.3ms$ & $11.2ms$ & $42.7ms$ & $121.8ms$ \\ \hline
	\end{tabular}
	\caption{Execution time of HGF on Nvidia GPU with increasing guidances channel number $n$, where the input is a gray image with increasing sizes, the channel number of guidance varies from $3$ to $9$. }
	\label{tab:time_arrayfire}
\end{table}

Box filtering takes up most of the run time of HGF. The area graph in Fig~\ref{fig:contribution} displays the execution time of box filtering in preparing matrices $\lambda \bm{E} + \bm{X}^T_{\mathfrak{p}} \bm{X}_{\mathfrak{p}} $~\eqref{eq:HGF_w1}, inverting these matrices and remainder computation in GF and HGF, where the green denotes the box filtering time, the orange presents matrix inverse time  and the horizontal axis is the channel number of guidance. We can observe that box filtering in HGF consumes more than half of the time. This inspires us to precompute the box filtering results $\bm{G}_{ij} = \mathcal{B}(\bm{G}_{i} \bm{G}_{j})$ of all pair-wise production image $\bm{G}_{i} \bm{G}_{j}$ in Eq~\eqref{eq:HGF_w4}~\eqref{eq:alpha} \eqref{eq:HGF_z2} to accelerate the computational speed of HGF further. Unlike HGF, matrix inverse operation in GF spends most of the time. So GF cannot take the same strategy to reduce its run time. At last, we note that an unignorable part of the overall execution time of HGF is spent by the ``remainder'' part denoted by the blue in Fig~\eqref{fig:contribution} and most time of this part is payed for box filtering as Eq~\eqref{eq:HGF_z2} involves $n$ box filters. Here $n$ is the channel number of the guidance.  


To demonstrate the actual power of our algorithm, we accelerate HGF by precomputing all-pair box filtering results $\bm{G}_{ij} = \mathcal{B}(\bm{G}_{i} \bm{G}_{j})$. Table~\ref{tab:time_arrayfire} reports the execution time tested on GPU with increasing guidance channel number. 
For guidance images of which the channel number is smaller than $7$, the performance is realtime. It is also worth noting that the performance of HGF scales well with the size of input image (\ie the execution time does not change with the data size). We contributes this ability to the parallel computing ability of HGF.

\begin{figure}[t]
	\centering	
	\begin{subfigure}[b]{0.48\linewidth}
		\includegraphics[width=\textwidth]{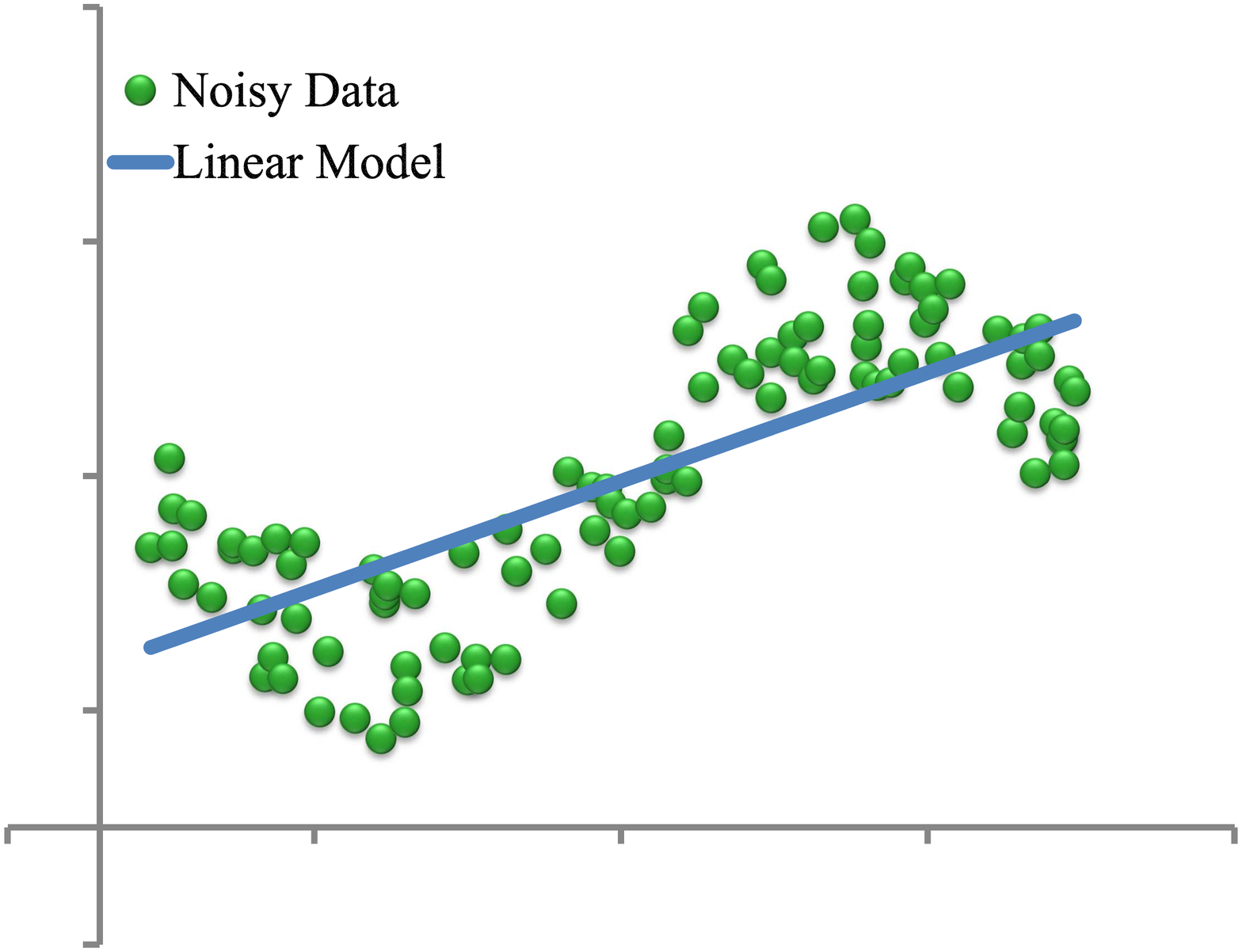}
		\caption{Linear model}
	\end{subfigure}	
	\begin{subfigure}[b]{0.48\linewidth}
		\includegraphics[width=\textwidth]{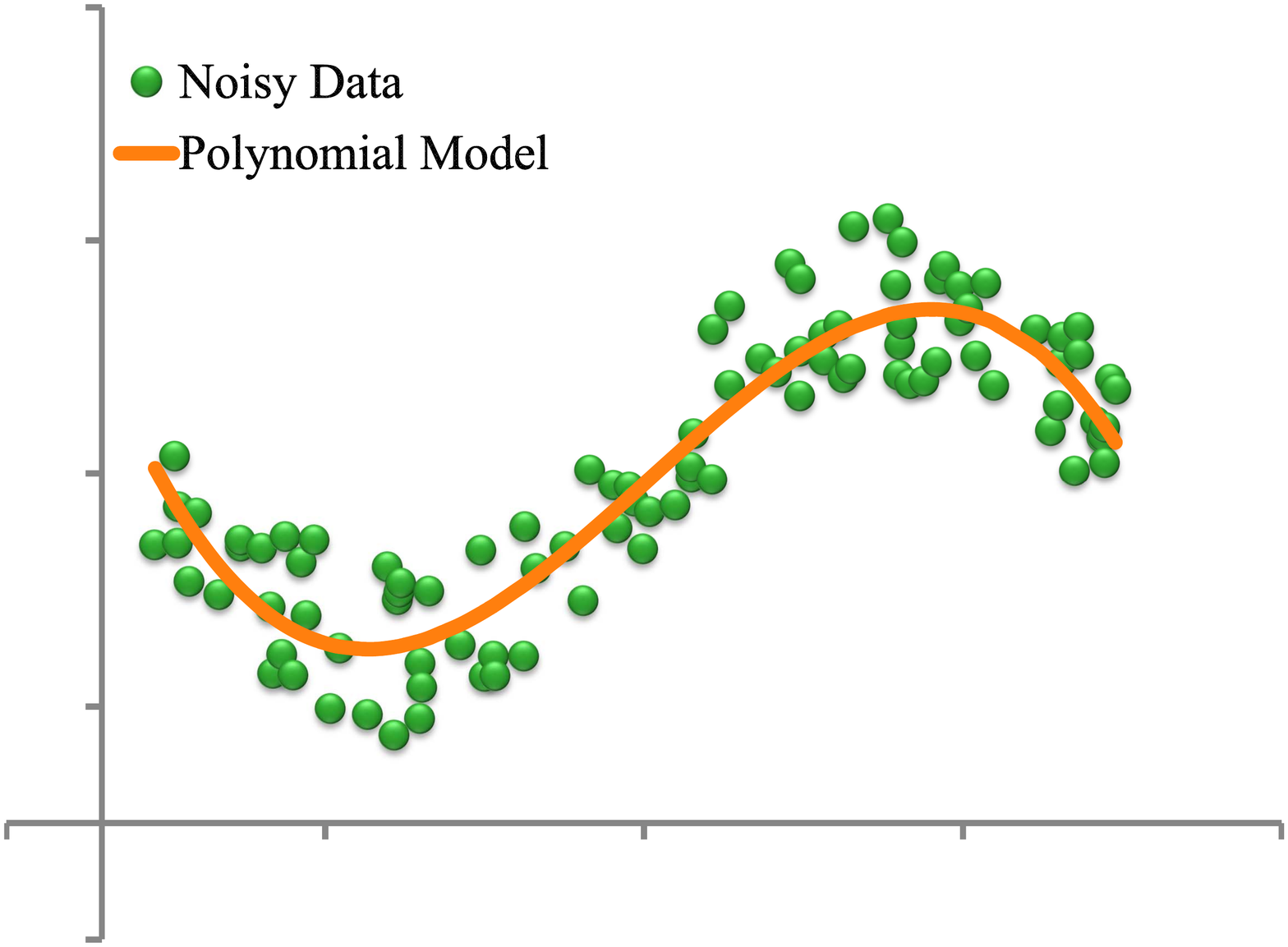}
		\caption{Polynomial model}
	\end{subfigure}	
\setcounter{figure}{4}
	\caption{\textbf{Fitting curves of linear model and nonlinear polynomial model on noisy data.} According to the fitting curves, the polynomial model is more flexible than the linear model. }
	\label{fig:curves}
\end{figure}

\begin{figure*}[t]
	\centering	
	\begin{subfigure}[b]{0.16\linewidth}
		\includegraphics[width=\textwidth]{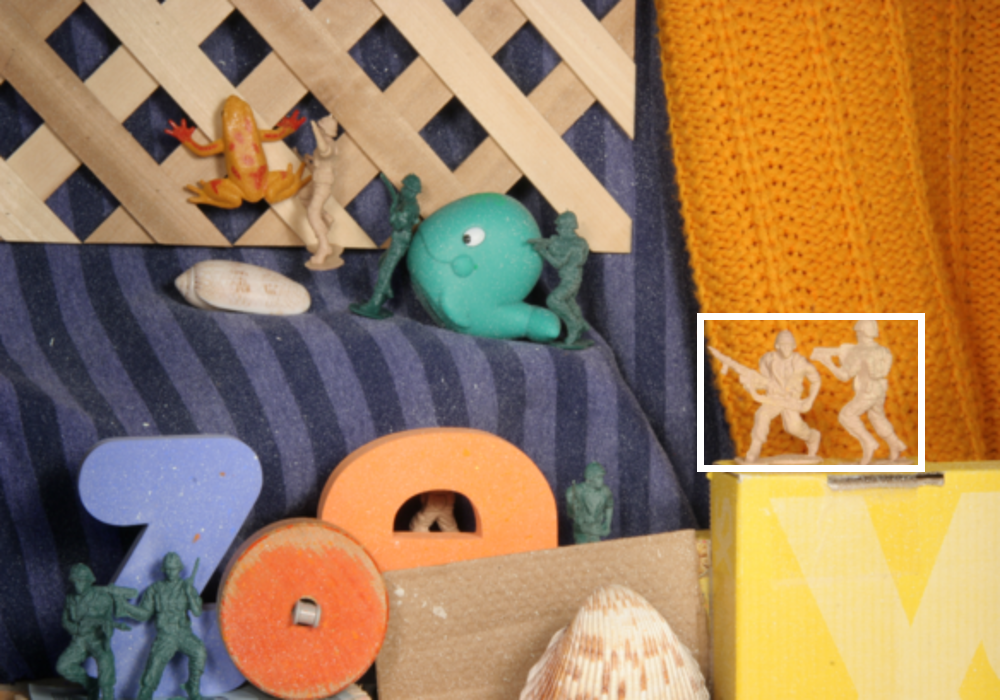}
	\end{subfigure}	
	\begin{subfigure}[b]{0.16\linewidth}
		\includegraphics[width=\textwidth]{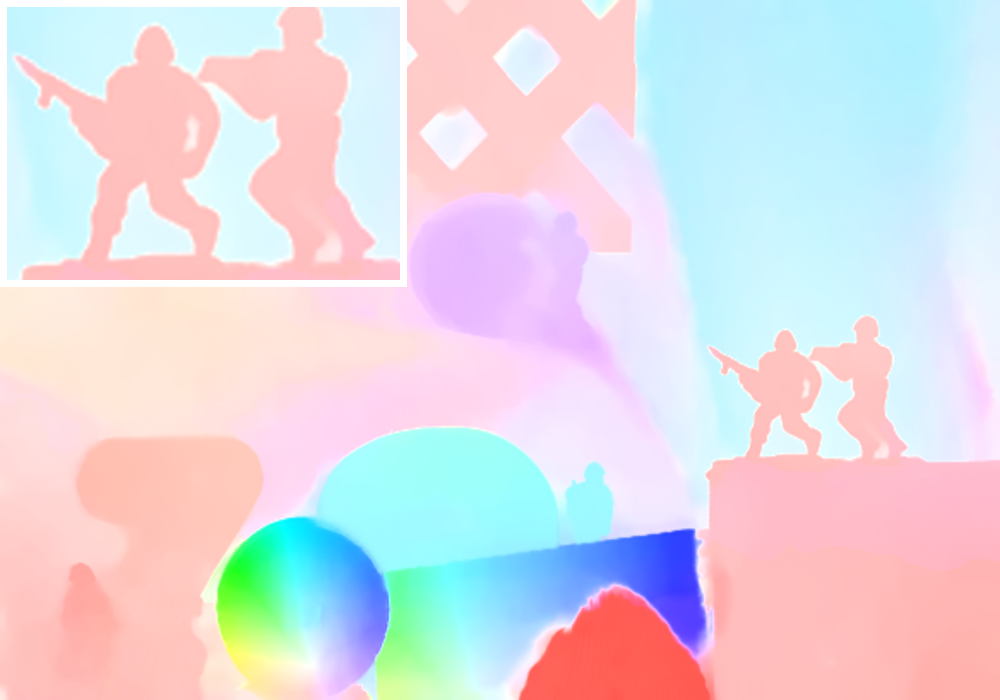}
	\end{subfigure}	
	\begin{subfigure}[b]{0.16\linewidth}
		\includegraphics[width=\textwidth]{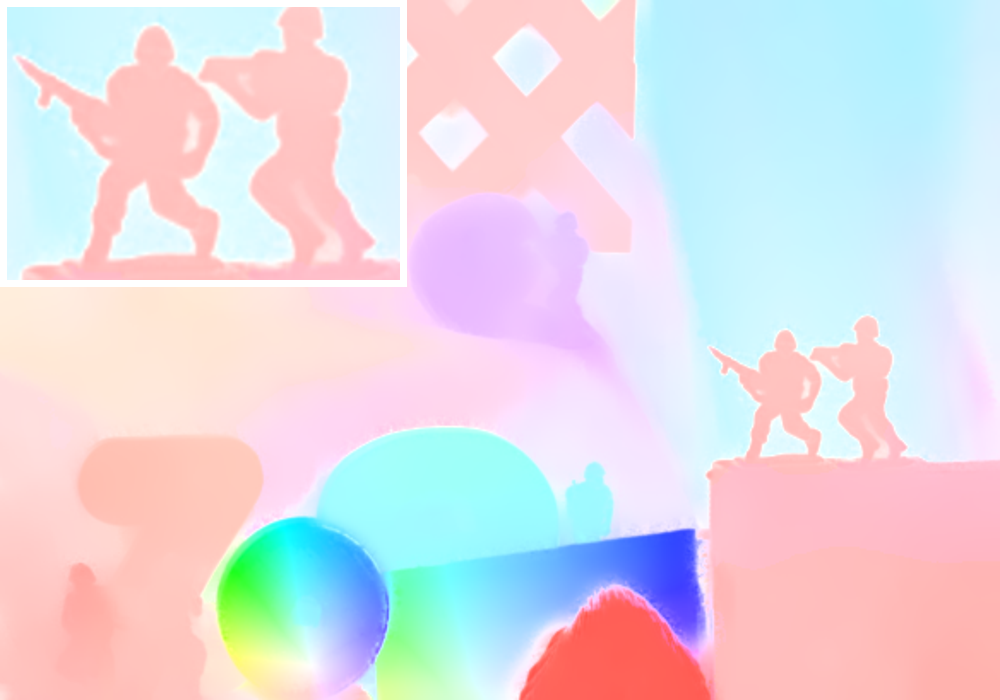}
	\end{subfigure}	
	\begin{subfigure}[b]{0.16\linewidth}
		\includegraphics[width=\textwidth]{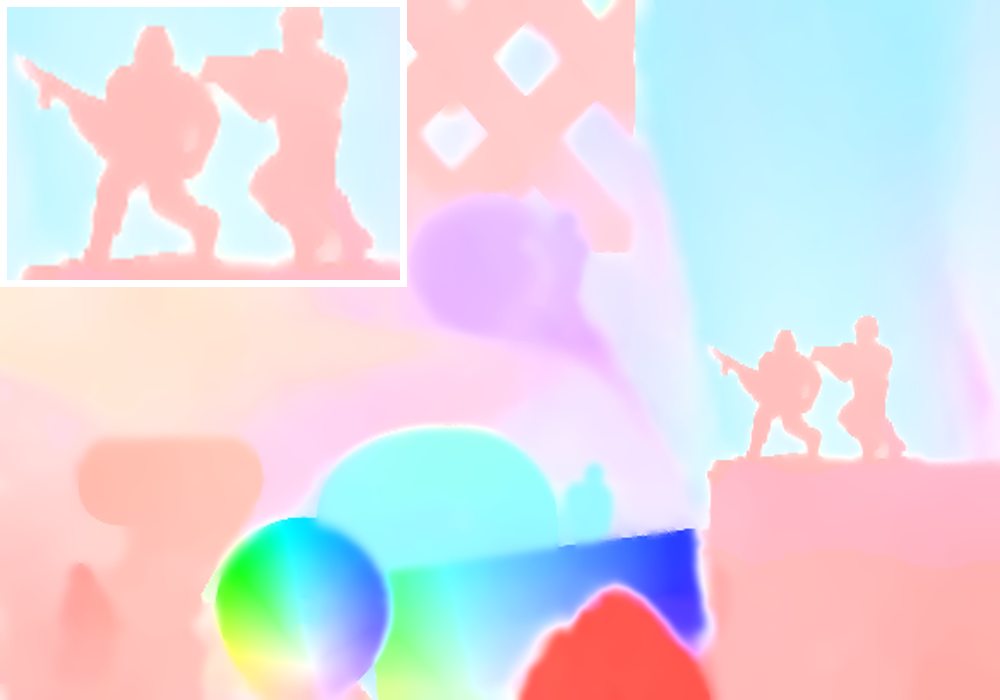}
	\end{subfigure}	
	\begin{subfigure}[b]{0.16\linewidth}
		\includegraphics[width=\textwidth]{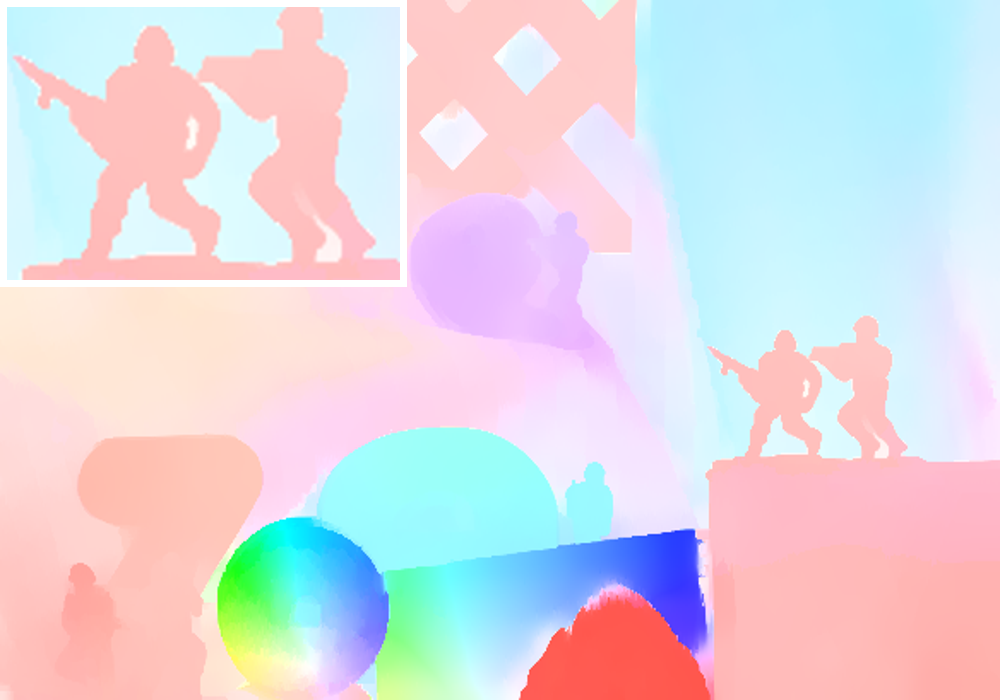}
	\end{subfigure}	
	\begin{subfigure}[b]{0.16\linewidth}
		\includegraphics[width=\textwidth]{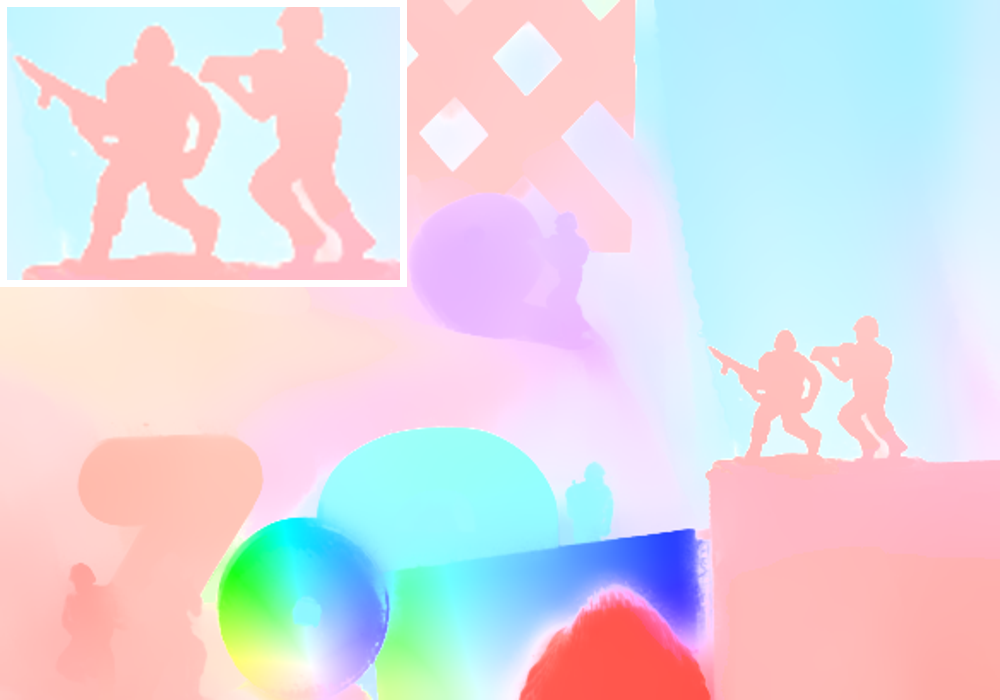}
	\end{subfigure}

	\begin{subfigure}[b]{0.16\linewidth}
		\includegraphics[width=\textwidth]{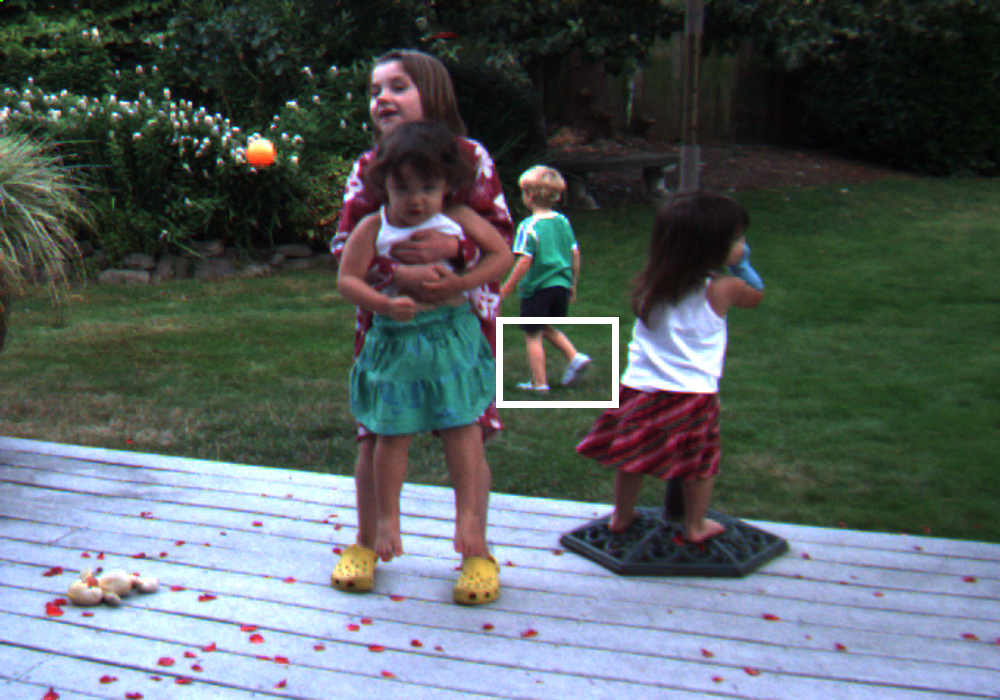}
		\caption{Input}
	\end{subfigure}	
	\begin{subfigure}[b]{0.16\linewidth}
		\includegraphics[width=\textwidth]{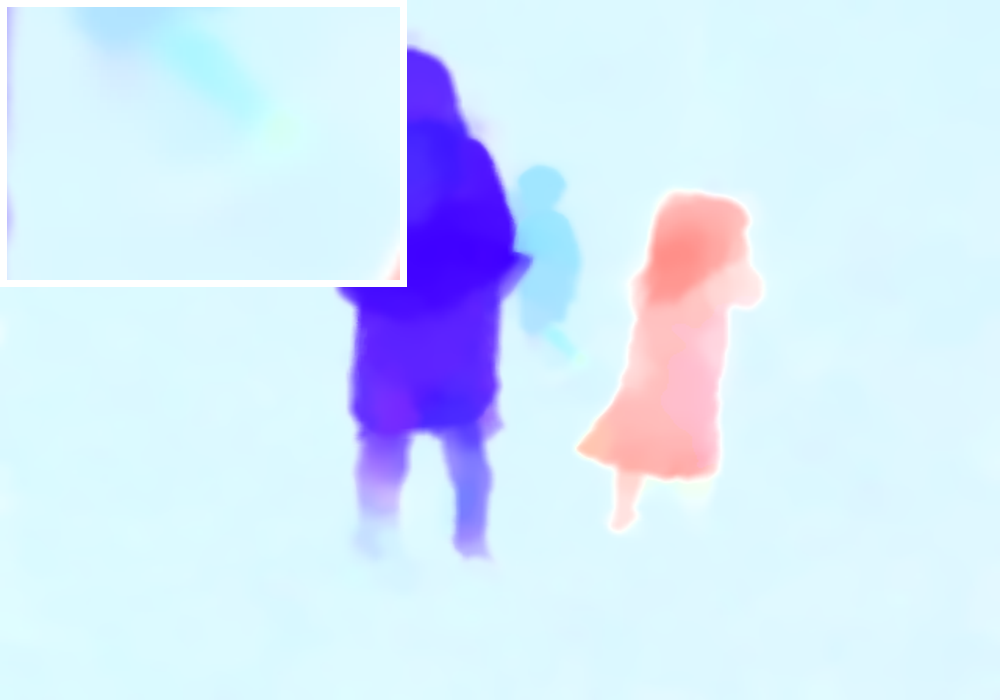}
		\caption{BF}
	\end{subfigure}	
	\begin{subfigure}[b]{0.16\linewidth}
		\includegraphics[width=\textwidth]{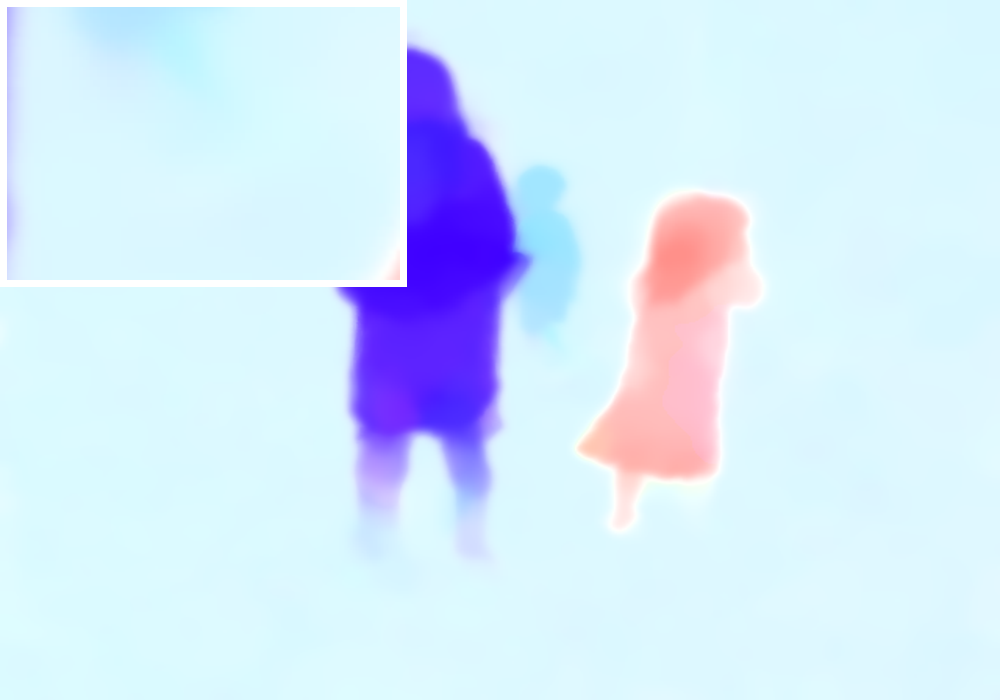}
		\caption{GF}
	\end{subfigure}	
	\begin{subfigure}[b]{0.16\linewidth}
		\includegraphics[width=\textwidth]{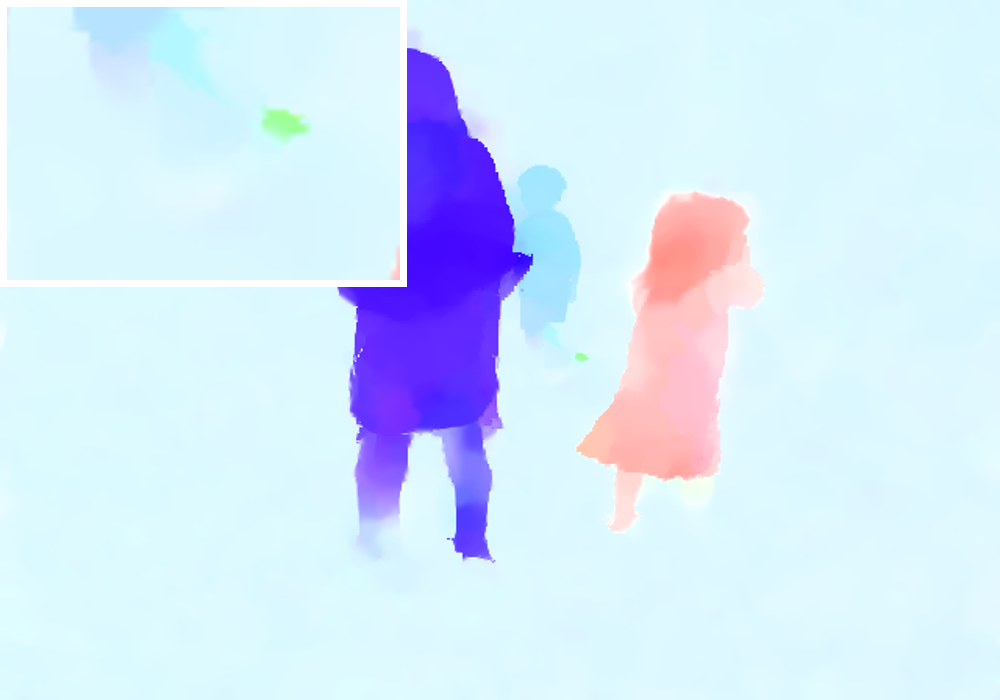}
		\caption{CLMF}
	\end{subfigure}	
	\begin{subfigure}[b]{0.16\linewidth}
		\includegraphics[width=\textwidth]{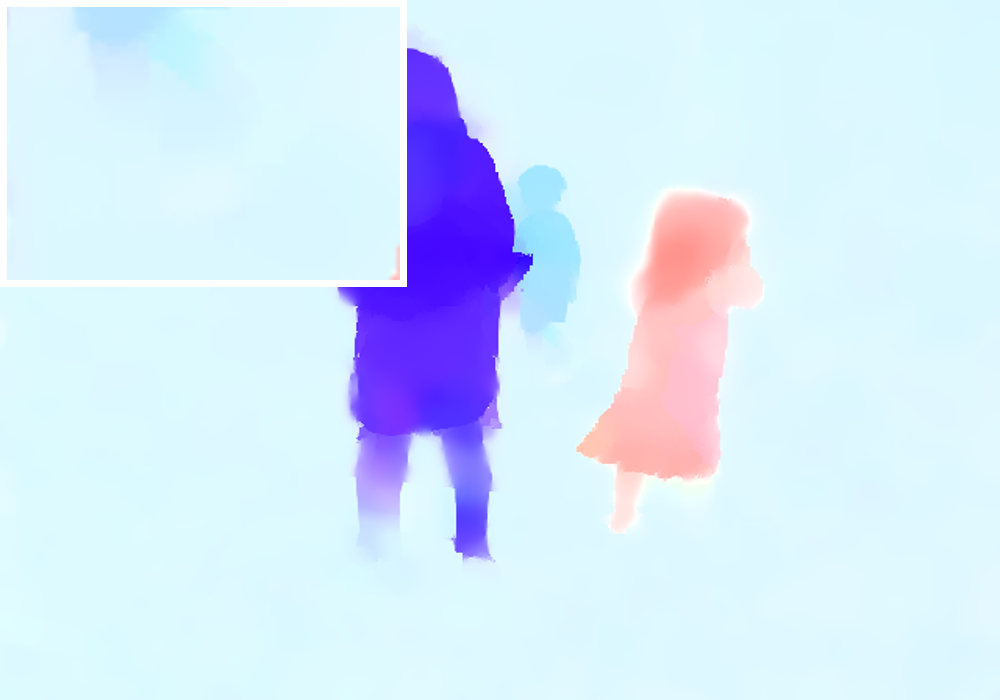}
		\caption{MLPA}
	\end{subfigure}	
	\begin{subfigure}[b]{0.16\linewidth}
		\includegraphics[width=\textwidth]{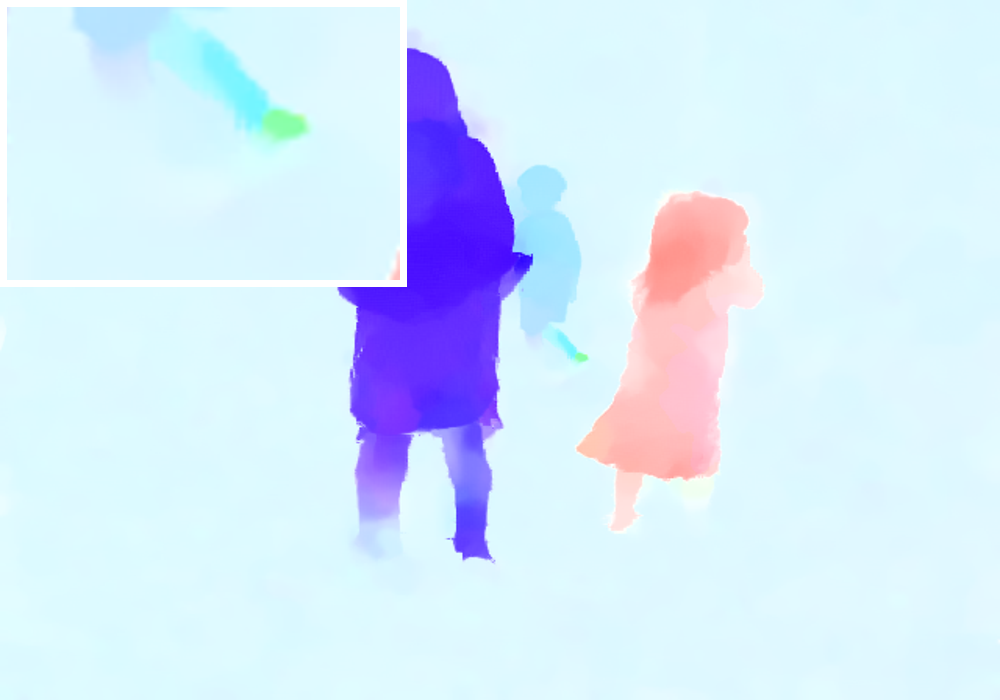}
		\caption{HGF}
	\end{subfigure}	
	\setcounter{figure}{6}
	\caption{\textbf{Optical Flow Estimation:} (a) is input color image.  (b) (c) (d) (e) and (f) demonstrate the stereo matching results of BF, GF, CLMF, MLPA and HGF, respectively, where the white box denotes the close-up region for visual comparison. }
	\label{fig:flow}
\end{figure*}

%
%
%
%
%

\section{Experiments and Applications}
\label{sec:exp}

To demonstrate the ability of our HGF, we apply it to three classic multi-label problems: stereo matching, optical flow and image segmentation. To show the robustness of our method, we use following, the same constant parameter setting to generate our results: $(\lambda = 0.05, r = 7)$, where $r$ indicates the radius of box window and $\lambda$ balances the two terms in Eq~\eqref{eq:HGF_linear_regression}. In the following experiments, each channel $\bm{G}_i$ of the polynomial guidance $\bm{G}$ is synthesized by the mapping $\bm{G}_{(i-1)d + j} = \bm{I}_i^{j}$ $(1 \leq j \leq 2)$ with a color input guidance $\bm{I}$ and total five edge-aware filters including BF~\cite{Tomasi_ICCV_1998}, GF~\cite{He_PAMI_2013}, CLMF~\cite{Lu_CVPR_2012}, MLPA~\cite{Tan_CVPR_2014} and our HGF are taken to perform comparison. For each application, we incorporate above five filters into the same framework to procedure final results. Experiments disclose our filter achieves the best result in terms of accuracy.

\vspace{0.1cm} \noindent \textbf{Nonlinearity \;}  We own the achievement of HGF to the synthesized polynomial guidance which introduces nonlinearity to our HGF. Fig~\ref{fig:curves} compares the polynomial model and linear model on an artificial dataset, which consists of a curve extracting from a nature image and strong noise added to every fifth datapoints. From the plotted fitting curves of linear model, we can observe that the linear model fails in the dataset. We thus believe that linear model  will fail too  while 
 smoothing highly corrupted cost slice in the following experiments


\vspace{0.1cm} \noindent \textbf{Stereo Matching \;}
We conduct experiments based on the cost volume filtering framework~\cite{Hosni_PAMI_2013}. This framework comprises cost volume computation, cost aggregation, disparity computation and post processing. Above five filters are employed for cost aggregation. According to the results reported in Fig~\ref{fig:stereo}, our method outperforms other aggregation filters and the run time is less than $800ms$.


\vspace{0.1cm}  \noindent \textbf{Optical Flow \;}
Xiao \etal~\cite{Xiao_ECCV_2006} separate the traditional one-step variational updating model into a two-step filtering-based updating model for optical flow estimation, where BF is proposed to substitute the original anisotropic diffusion process. In this framework, we substitute original BF with  GF, CLMF, MLPA and  HGF to test the performance of these filters. According to Fig~\ref{fig:flow}, our method can detect tiny structure in the optical flow images within $900ms$.


\vspace{0.2cm} \noindent \textbf{Image Segmentation \;}
We also take the cost volume filtering framework~\cite{Hosni_PAMI_2013} to show that HGF performs well for image segmentation, where the labels encode whether a pixel belongs to the foreground or background and cost computation is same to Hosni \etal~\cite{Hosni_PAMI_2013}. Fig~\ref{fig:segment} proves that our HGF is able to distinguish foreground from background much better and the run time is $120ms$.

\begin{figure}
	\setlength{\tabcolsep}{0cm}
	\begin{tabular}{cccc}
		\multirow{2}{*}{ \parbox[c]{0.3\linewidth}{ \vspace{-0.8cm} 
				\begin{subfigure}[b]{\linewidth}
				\includegraphics[width=\textwidth]{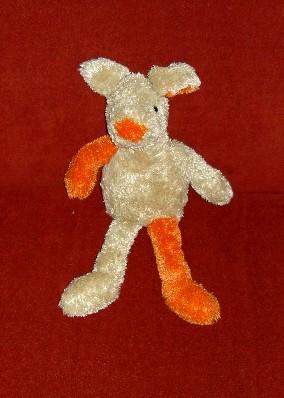}
				\caption{Input}
		\end{subfigure}	} } 
	    & 
		\parbox[c]{0.2\linewidth}{
	    \begin{subfigure}[b]{0.98\linewidth}
			\includegraphics[width=\textwidth]{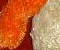}
			\caption{Close-up}
		\end{subfigure} }
	    & 
	    \parbox[c]{0.2\linewidth}{
	    \begin{subfigure}[b]{0.98\linewidth}
	   		\includegraphics[width=\textwidth]{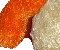}
	   		\caption{BF}
	    \end{subfigure} }
        &
        \parbox[c]{0.2\linewidth}{ 
        	\begin{subfigure}[b]{0.98\linewidth}
        		\includegraphics[width=\textwidth]{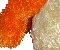}
        		\caption{GF}
        \end{subfigure} } \\
		& 
		\parbox[c]{0.2\linewidth}{ \vspace{0.15cm}
			\begin{subfigure}[b]{0.98\linewidth}
				\includegraphics[width=\textwidth]{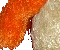}
				\caption{CLMF}
		\end{subfigure} }
	    & 
	    \parbox[c]{0.2\linewidth}{ \vspace{0.15cm}
	    	\begin{subfigure}[b]{0.98\linewidth}
	    		\includegraphics[width=\textwidth]{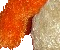}
	    		\caption{MLPA}
	    \end{subfigure} }
        & 
        \parbox[c]{0.2\linewidth}{ \vspace{0.15cm}
        	\begin{subfigure}[b]{0.98\linewidth}
        		\includegraphics[width=\textwidth]{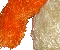}
        		\caption{HGF}
        \end{subfigure} }
	\end{tabular}
	\caption{\textbf{Segmentation results:} (a) is input image, where black box denotes bounding box and white box presents the close-up region. (b) is the close-up of (a). (c) (d) (e) (f) (g) are the segmentation results of BF, GF, CLMF, MLPA and HGF, respectively.  }
	\label{fig:segment}
\end{figure}

\section{Conclusion}

This paper presented an effective guided image filter for multi-label problem. We own the power to the nonlinearity introduced by the synthesized polynomial guidance. A side effect of the nonlinearity model is that it inevitably increases the running cost as we have to invert a bulk of large matrices. Fortunately, our new designed hardware-efficient matrix inverse algorithm can reduce the run time significantly by the help of our effective matrix inverse technique. We believe our filter will greatly profit building efficient computer vision systems in other multi-label problems.

\section{Acknowledgments}

This work was supported in part by the National Key Research and Development Program of China (Grant No. 2016YFB1001001) and the National Natural Science Foundation of China (Grant No. 61620106003, 61522203, 61571046).

\newpage 
{\small
\bibliographystyle{ieee}
\bibliography{egbib}
}

\end{document}